\documentclass{article}

\PassOptionsToPackage{numbers, compress}{natbib}

\usepackage[preprint]{neurips_2026}

\usepackage{hyperref}
\usepackage{url}
\usepackage{enumitem}
\usepackage{tcolorbox}

\usepackage{amsmath,amssymb,amsthm}

\usepackage{amsmath,amsfonts,bm}

\def\secref#1{section~\ref{#1}}
\def\Secref#1{Section~\ref{#1}}

\def\eqref#1{equation~\ref{#1}}
\def\Eqref#1{Equation~\ref{#1}}

\def\algref#1{algorithm~\ref{#1}}

\def\1{\bm{1}}

\DeclareMathAlphabet{\mathsfit}{\encodingdefault}{\sfdefault}{m}{sl}
\SetMathAlphabet{\mathsfit}{bold}{\encodingdefault}{\sfdefault}{bx}{n}

\newcommand{\E}{\mathbb{E}}

\newcommand{\R}{\mathbb{R}}

\DeclareMathOperator*{\argmin}{arg\,min}

\providecommand{\R}{\mathbb{R}} %

\providecommand{\xx}{\mathbf{x}}

\providecommand{\cN}{\mathcal{N}}

\newenvironment{talign*}
{\csname align*\endcsname}
{\endalign}

\usepackage[utf8]{inputenc}         %
\usepackage[T1]{fontenc}            %
\usepackage{url}                    %
\usepackage{booktabs}               %
\usepackage{tabularx}               %
\usepackage{amsfonts}               %
\usepackage{nicefrac}               %
\usepackage{microtype}              %
\usepackage{xcolor}                 %
\usepackage{algorithm}
\usepackage{algpseudocode}
\usepackage{graphicx}
\usepackage{subcaption}
\usepackage[flushleft]{threeparttable}
\usepackage{float}
\usepackage{multirow}
\usepackage{makecell}
\usepackage{xspace}
\usepackage{enumitem}
\usepackage[font=small]{caption}
\usepackage{autobreak}
\usepackage{sidecap}
\usepackage{wrapfig}
\usepackage{bbding}
\usepackage[toc, page, header]{appendix}
\usepackage{tikz}
\usetikzlibrary{calc,trees,positioning,arrows,chains,shapes.geometric,%
    decorations.pathreplacing,decorations.pathmorphing,shapes,%
    matrix,shapes.symbols}
\usepackage{xcolor}
\usepackage{pifont}
\usepackage{mdframed}
\usepackage{colortbl}

\providecommand{\halfcheck}{\mbox{\ding{52}\rotatebox[origin=c]{-9.2}{\kern-0.7em\ding{55}}}}

\usepackage{tcolorbox}
\tcbuselibrary{skins, breakable, theorems}
\usepackage{empheq}
\usepackage{arydshln}
\usepackage{bm}
\usepackage[capitalize]{cleveref}
\usepackage{listings}

\hypersetup{
    colorlinks=true,
    linkcolor=blue,
    citecolor=blue,
    urlcolor=blue
}

\definecolor{coral}{RGB}{255,127,80}
\definecolor{darkgreen}{RGB}{0,100,0}
\definecolor{darkyellow}{RGB}{204,153,0}
\definecolor{salmon}{RGB}{250,128,114}
\definecolor{darkred}{RGB}{150,0,0}
\newcommand{\darkredtext}[1]{{\color{darkred}#1}}

\definecolor{eqbg}{gray}{0.95}
\definecolor{indomaincolor}{rgb}{0.9, 0.95, 1.0}
\definecolor{oodcolor}{rgb}{1.0, 0.9, 0.9}

\newcommand{\transparentgray}[1]{%
    \tikz[baseline=(X.base)] \node[fill=gray, fill opacity=0.1, text opacity=1, inner sep=2pt, outer sep=0pt] (X) {#1};%
}

\newcommand{\transparentyellow}[1]{%
    \tikz[baseline=(X.base)] \node[fill=yellow, fill opacity=0.1, text opacity=1, inner sep=2pt, outer sep=0pt] (X) {#1};%
}

\renewcommand{\secref}[1]{\hyperref[#1]{\darkredtext{Sec.~\ref*{#1}}}}
\renewcommand{\Secref}[1]{\hyperref[#1]{\darkredtext{Sec.~\ref*{#1}}}}
\providecommand{\thmref}[1]{\hyperref[#1]{\darkredtext{Thm.~\ref*{#1}}}}
\providecommand{\defref}[1]{\hyperref[#1]{\transparentgray{Definition~\ref*{#1}}}}
\providecommand{\propref}[1]{\hyperref[#1]{\darkredtext{Prop.~\ref*{#1}}}}
\providecommand{\assumpref}[1]{\hyperref[#1]{\darkredtext{Assump.~\ref*{#1}}}}
\providecommand{\remarkref}[1]{\hyperref[#1]{\transparentyellow{Remark~\ref*{#1}}}}
\providecommand{\conjref}[1]{\hyperref[#1]{\darkredtext{Conj.~\ref*{#1}}}}
\providecommand{\lemref}[1]{\hyperref[#1]{\darkredtext{Lem.~\ref*{#1}}}}
\providecommand{\corref}[1]{\hyperref[#1]{\darkredtext{Cor.~\ref*{#1}}}}
\providecommand{\noteref}[1]{\hyperref[#1]{\darkredtext{Nota.~\ref*{#1}}}}
\providecommand{\claimref}[1]{\hyperref[#1]{\darkredtext{Clm.~\ref*{#1}}}}
\providecommand{\algref}[1]{\hyperref[#1]{\darkredtext{Alg.~\ref*{#1}}}}
\providecommand{\algmref}[1]{\hyperref[#1]{\darkredtext{Alg.~\ref*{#1}}}}
\providecommand{\figref}[1]{\hyperref[#1]{\darkredtext{Fig.~\ref*{#1}}}}
\providecommand{\tabref}[1]{\hyperref[#1]{\darkredtext{Tab.~\ref*{#1}}}}
\providecommand{\appref}[1]{\hyperref[#1]{\darkredtext{App.~\ref*{#1}}}}

\newtheoremstyle{professional}
{10pt} %
{10pt} %
{\itshape} %
{} %
{\bfseries} %
{.} %
{.5em} %
{} %

\theoremstyle{professional}
\newtheorem{myth}{Theorem}[section]
\newtheorem{myprop}[myth]{Proposition}
\newtheorem{mylem}[myth]{Lemma}
\newtheorem{mycor}[myth]{Corollary}
\newtheorem{mydef}[myth]{Definition}
\newtheorem{myassump}[myth]{Assumption}
\newtheorem{myrem}[myth]{Remark}
\newtheorem{myhyp}[myth]{Hypothesis}
\newtheorem{myconj}[myth]{Conjecture}
\newtheorem{mynota}[myth]{Notation}
\newtheorem{myclaim}[myth]{Claim}
\newtheorem{myprob}[myth]{Problem}
\newtheorem{myobs}[myth]{Observation}

\tcbset{
    thmbox/.style={
            enhanced,
            breakable,
            sharp corners,
            boxrule=0pt,
            leftrule=3pt,
            top=0pt,
            bottom=0pt,
            left=5pt,
            right=5pt,
            before skip=10pt,
            after skip=10pt,
        }
}

\newenvironment{theorem}{\begin{tcolorbox}[thmbox, colback=red!5!white, colframe=red!75!black]\begin{myth}}{\end{myth}\end{tcolorbox}}

\newenvironment{lemma}{\begin{tcolorbox}[thmbox, colback=cyan!5!white, colframe=cyan!75!black]\begin{mylem}}{\end{mylem}\end{tcolorbox}}

\newenvironment{assumption}{\begin{tcolorbox}[thmbox, colback=green!5!white, colframe=green!75!black]\begin{myassump}}{\end{myassump}\end{tcolorbox}}

\newtheorem{innernote}{Note}

\newtheorem{innerexercise}{Exercise}

\definecolor{codegreen}{rgb}{0,0.6,0}
\definecolor{codegray}{rgb}{0.5,0.5,0.5}
\definecolor{codepurple}{rgb}{0.58,0,0.82}
\definecolor{backcolour}{rgb}{0.95,0.95,0.92}

\lstdefinestyle{mystyle}{
    backgroundcolor=\color{backcolour},
    commentstyle=\color{codegreen},
    keywordstyle=\color{magenta},
    numberstyle=\tiny\color{codegray},
    stringstyle=\color{codepurple},
    basicstyle=\ttfamily\footnotesize,
    breakatwhitespace=false,
    breaklines=true,
    captionpos=b,
    keepspaces=true,
    numbers=left,
    numbersep=5pt,
    showspaces=false,
    showstringspaces=false,
    showtabs=false,
    tabsize=2
}
\newcommand{\methodname}{\textsc{FlowCPO}}
\newcommand{\method}{\methodname\xspace}

\title{FlowCPO: A Unified Divergence View of Preference Alignment for Flow Models}

\author{
Yansen Han$^{1,2,\ast}$ \quad
Shengyi Liao$^{3,\ast}$ \quad
Peng Sun$^{1,2}$ \quad
Deyuan Liu$^{1}$ \quad
Yuanxing Zhang$^{3}$ 
\\[0.5em]
\textbf{Pengfei Wan$^{3}$} \quad 
\textbf{Tao Lin$^{1,\dagger}$}
\\[0.5em]
$^\ast$Equal contribution \quad
$^\dagger$Corresponding author
\\[0.25em]
$^1$Westlake University \quad
$^2$Zhejiang University \quad
$^3$Kling Team, Kuaishou Technology
}

\date{}
\hypersetup{
  pdftitle={FlowCPO: A Unified Divergence View of Preference Alignment for Flow Models},
  pdfauthor={Yansen Han, Shengyi Liao, Peng Sun, Deyuan Liu, Yuanxing Zhang, Pengfei Wan, Tao Lin}
}

\begin{document}

\maketitle

\begin{abstract}
    Preference alignment for flow and diffusion models now spans online reinforcement learning and offline preference optimization, but the relation between these methods remains unclear.
    In particular, existing forward-process alignment methods require fresh samples from the current model, while offline methods based on fixed preference pairs rely primarily on positive-only fine-tuning or DPO-style likelihood-ratio surrogates.
    We organize these approaches through a divergence-based framework and introduce \textbf{\method}, an offline forward-KL objective that uses both preferred and dispreferred samples without online rollouts.
    For linear interpolation, we show under explicit regularity conditions that the forward-KL objective is bounded by a contrastive flow matching loss, yielding a tractable surrogate on fixed data.
    We further show that this loss is nonnegative, whereas the signed regression loss of simplified FlowDPO can be unbounded below.
    In the in-domain setting, \method achieves higher mean GenEval and OCR scores than the evaluated baselines, reaching 0.84 and 0.87 versus 0.81 and 0.74 for FlowDPO at CFG 3.0. In the out-of-domain setting, the results are mixed, with the best GenEval result but lower reward scores than RFT on several metrics.
\end{abstract}

\section{Introduction}
Recent advancements in flow matching and diffusion-based generative modeling~\citep{ho2020denoising, lipman2022flow, albergo2022building, sohl2015deep, luo2022understanding, holderrieth2025introduction} have facilitated the development of practical preference alignment algorithms for continuous generative spaces~\citep{liu2025flow,zheng2025diffusionnft,liu2025improving,xue2025advantage,bergmeister2026reinforce}.
These efforts have yielded diverse alignment paradigms, ranging from online reinforcement learning frameworks like FlowGRPO~\citep{liu2025flow}, DiffusionNFT~\citep{zheng2025diffusionnft}, AWM~\citep{xue2025advantage} and RAM~\citep{bergmeister2026reinforce} to offline preference optimization methods such as FlowDPO~\citep{liu2025improving}.
In contrast to alignment in discrete autoregressive models, these continuous-time methods operate directly on sampling trajectories and velocity fields rather than token-level probabilities.
While empirical results are promising, the relations among these methods remain difficult to compare.

\begin{wraptable}[8]{r}{0.5\textwidth}
    \centering
    \vspace{-1em}
    \caption{\small
        \textbf{Objective-level view of flow-based preference alignment.} The two axes separate the sampling regime from the direction of the idealized KL objective.
    }
    \vspace{-0.5em}
    \resizebox{0.5\textwidth}{!}{%
        \begin{tabular}{lccc}
            \toprule
            \textbf{Divergence} & \textbf{Online} & \textbf{Offline}        \\
            \midrule
            Reverse-KL          & FlowGRPO        & FlowDPO                 \\
            Forward-KL          & DiffusionNFT    & \textbf{\method (ours)} \\
            \bottomrule
        \end{tabular}
    }
    \label{tab:method_taxonomy}
\end{wraptable}

Existing flow-based alignment methods can be organized along two axes (\tabref{tab:method_taxonomy}): the \emph{sampling regime} (online versus offline) and \emph{the direction of the idealized KL objective}.
Under this view, FlowGRPO~\citep{liu2025flow} and FlowDPO~\citep{liu2025improving} represent online and offline instances of the reverse-KL branch, whereas DiffusionNFT~\citep{zheng2025diffusionnft} provides a forward-process objective but refreshes its training samples online.
The unresolved problem is therefore the offline forward-KL objective using both preferred and dispreferred target samples: a forward-KL expectation can use fixed target samples, but its log-likelihood is intractable for flow matching.
An offline method needs a justified bridge from this distribution-level objective to velocity-field regression on fixed preferred and dispreferred data.

We address this gap by establishing a unified divergence-based framework for flow-based preference alignment and deriving \textbf{Flow Contrastive Preference Optimization (\method)}.
Starting from forward-KL terms over preferred and dispreferred target distributions, we use a flow matching bound for linear interpolation to obtain a contrastive regression surrogate.
The resulting loss can be estimated from a fixed preference dataset and therefore requires no sampling from the model during training.

We also connect \method to simplified FlowDPO (\appref{app:flowdpo_ema_connection}).
Simplified FlowDPO subtracts the regression error on dispreferred samples and can be unbounded below, whereas \method adds two nonnegative regression errors.
With equal branch weights ($\lambda=1$), our loss decomposition identifies corrections determined by the difference between the current model and its exponential moving average (EMA).

Empirically, we evaluate our approach under a strictly offline protocol. In the in-domain setting, where preference pairs are generated by the reference model, \method attains higher mean GenEval and OCR scores than RFT and FlowDPO while remaining strong on general-preference metrics. In the out-of-domain setting, the result is more mixed: \method gives the best GenEval score but RFT is stronger on several reward metrics. The overall offline optimization pipeline is illustrated in \figref{fig:flowcpo}. \textbf{Our contributions are threefold:}
\begin{itemize}[leftmargin=12pt, nosep]
    \item We provide a unified divergence-based framework that separates the online/offline sampling regime from the direction of the idealized KL objective, clarifying the assumptions behind connections among existing flow-based alignment methods.
    \item We derive \method, a nonnegative contrastive surrogate for offline forward-KL alignment. Under stated regularity conditions for linear interpolation, its regression objective bounds the branch-wise forward-KL objective and is estimable from fixed preference data. We also derive its loss decomposition relative to simplified FlowDPO.
    \item We evaluate \method with two types of offline data, showing higher GenEval and OCR scores than FlowDPO in the in-domain setting and mixed results in the out-of-domain setting.
\end{itemize}

\begin{figure}[t]
    \centering
    \includegraphics[width=1\textwidth]{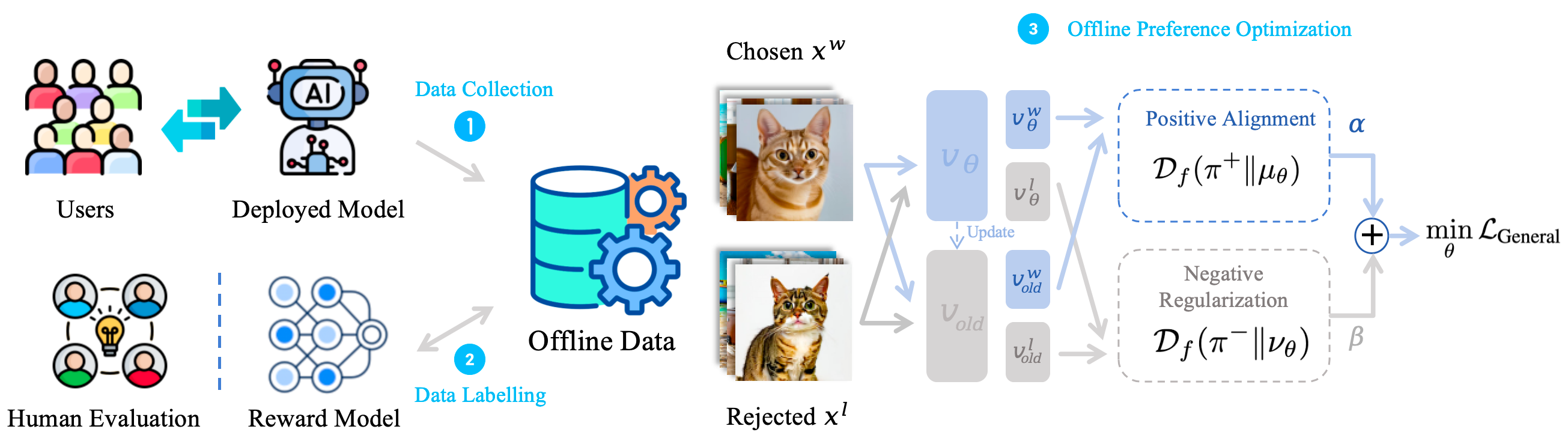}
    \caption{\small
        \textbf{\method uses a fixed preference dataset to train two coupled flow branches.} A frozen generator supplies preferred/dispreferred pairs $\mathcal{D}=\{(c,\mathbf{x}_0^w,\mathbf{x}_0^l)\}$; the positive branch matches preferred samples, while a mirrored branch matches dispreferred samples. Training uses only noised versions of these fixed endpoints and requires no online rollout.
    }
    \label{fig:flowcpo}
\end{figure}

\section{Related Work}
\label{sec:related_work}

\paragraph{Online alignment with policy-generated samples.}
Online methods improve a generator using samples drawn from its current or recent policy.
ReFL~\citep{xu2023imagereward} uses reward feedback, DDPO~\citep{black2024ddpo} and DPOK~\citep{fan2023dpok} estimate trajectory policy gradients, and DRaFT~\citep{clark2024draft} and AlignProp~\citep{prabhudesai2024alignprop} differentiate through the sampler.
FlowGRPO~\citep{liu2025flow} optimizes stochasticized flow trajectories, with subsequent refinements to temporal allocation and trajectory reuse~\citep{he2025tempflow,li2025mixgrpo,li2025branchgrpo,ding2025treegrpo}. SPO~\citep{liang2025spo} generates candidates at each denoising step and selects preference pairs with a step-aware model.
DiffusionNFT~\citep{zheng2025diffusionnft}, AWM~\citep{xue2025advantage}, and RAM~\citep{bergmeister2026reinforce} use regression-style objectives on noised on-policy samples.

\paragraph{Offline alignment with fixed preference data.}
Fixed-data methods differ in how they use preferred and dispreferred samples.
D3PO~\citep{yang2024d3po} extends pairwise learning to diffusion trajectories; DiffusionDPO~\citep{wallace2024diffusion} uses reference-relative denoising errors, with temporal weighting in Dense Reward~\citep{yang2024dense_reward} and a flow-model extension in FlowDPO~\citep{liu2025improving}.
DMPO~\citep{li2025divergence} develops a reverse-KL objective, whereas MaPO~\citep{hong2026mapo} learns preference margins without a reference model.
RFT~\citep{xiong2025minimalist,chen2025bridging} fits only preferred samples.
\method instead uses both sides of fixed preference pairs in a coupled, two-branch forward-KL objective, with a flow-matching upper-bound surrogate under stated conditions.

\paragraph{Divergence-based alignment in language models.}
LLM alignment also treats divergence and utility as design choices.
$f$-DPG~\citep{go2023aligning} generalizes distribution matching, while $f$-DPO~\citep{wang2023beyond} varies the divergence constraint for preference learning.
GPO~\citep{tang2024generalized} relates DPO~\citep{rafailov2023direct}, IPO~\citep{azar2024general}, and SLiC~\citep{zhao2023slic} through convex pairwise losses. $f$-PO~\citep{han2024f} unifies divergence-minimization formulations including DPO and EXO~\citep{ji2024towards}.

\section{Preliminaries}
\label{sec:preliminaries}
We collect notation and background needed for our divergence-based view of flow alignment.
Specifically, \secref{subsec:diffusion_fm} recalls diffusion models and flow matching under linear interpolation.
\secref{subsec:rlhf_dpo} summarizes RLHF and DPO in the discrete policy formulation, which we then lift to continuous-time models in \secref{subsec:pref_opt_continuous}.
Finally, \secref{subsec:kl_divergences} reviews forward and reverse KL divergences.

\subsection{Diffusion Models and Flow Matching} \label{subsec:diffusion_fm}
We first briefly review the two continuous-time generative paradigms used throughout the paper.
Both start from a simple latent variable, typically Gaussian noise $\epsilon \sim \mathcal{N}(0, I)$, and define intermediate states $\xx_t$ for $t \in [0,1]$ that connect data and noise.

In diffusion models~\citep{ho2020denoising, sohl2015deep}, one specifies a forward noising process $q(\xx_t \mid \xx_0)$ and trains a network $\epsilon_\theta(\xx_t, c, t)$, or equivalently a score model, to predict the noise added to the clean sample $\xx_0$.
Generation then approximately reverses this process from noise to data through a reverse-time SDE or the associated probability flow ODE.
Flow matching~\citep{lipman2022flow, albergo2022building} instead directly learns the vector field of a prescribed probability path.
In this paper we use the linear interpolation
\begin{equation}
    \label{eq:linear_bridge_prelim}
    \xx_t = (1-t) \cdot \xx_0 + t \cdot \epsilon, \qquad \epsilon \sim \mathcal{N}(0, I) \,,
\end{equation}
so that $\xx_0$ is the clean sample at $t=0$ and $\xx_1=\epsilon$ is the noise sample at $t=1$.
The corresponding conditional target velocity is
\begin{equation}
    \label{eq:fm_target_velocity_prelim}
    u_t(\xx_t \mid \xx_0) = \frac{\mathrm{d}\xx_t}{\mathrm{d}t} = \epsilon - \xx_0 \,.
\end{equation}
Flow matching trains a velocity field $v_\theta(\xx_t, c, t)$ to regress to this target, typically by minimizing
\begin{equation}
    \label{eq:fm_loss_prelim}
    \mathcal{L}_{\text{FM}}(\theta)
    = \mathbb{E}_{c, \xx_0, t, \epsilon}
    \left[
        \|v_\theta(\xx_t, c, t) - u_t(\xx_t \mid \xx_0)\|_2^2
        \right] \,.
\end{equation}
Although diffusion and flow matching use different parameterizations, they are closely connected: under suitable probability paths, the denoising model, score function, and velocity field describe the same underlying transport from noise to data~\citep{lai2025principles,holderrieth2025introduction,sun2025unified,lipman2024flow}.

\subsection{RLHF and Direct Preference Optimization (DPO)} \label{subsec:rlhf_dpo}
Reinforcement Learning from Human Feedback (RLHF) \cite{ouyang2022training} typically aligns a generative policy $\pi_\theta(\xx_0|c)$ by maximizing expected reward while regularizing deviation from a reference policy, where $\xx_0$ is the generated content and $c$ is the condition.
We use $r(\xx_0, c)$ to denote the reward function, $\xx_0^w$ and $\xx_0^l$ to denote the preferred and dispreferred contents, respectively.
DPO \cite{rafailov2023direct, sun2025solopo, tang2024generalized} bypasses the explicit reward modeling step by directly optimizing the policy using a closed-form solution to the KL-constrained reward maximization problem.
Given a dataset $\mathcal{D} = \{(c, \xx_0^w, \xx_0^l)\}$ of preferred ($\xx_0^w$) and dispreferred ($\xx_0^l$) contents, DPO minimizes the following negative log-likelihood loss:
\looseness=-1 
\begin{equation}
    \label{eq:dpo}
    \mathcal{L}_{\text{DPO}}(\theta) = -\mathbb{E}_{(c, \xx_0^w, \xx_0^l) \sim \mathcal{D}} \left[ \log \sigma \left( \beta \log \frac{\pi_\theta(\xx^w_0|c)}{\pi_{\text{ref}}(\xx^w_0|c)} - \beta \log \frac{\pi_\theta(\xx_0^l|c)}{\pi_{\text{ref}}(\xx_0^l|c)} \right) \right] \,,
\end{equation}
where $\pi_{\text{ref}}$ is the frozen reference policy and $\beta$ is a temperature parameter controlling the strength of the KL constraint.

\subsection{Preference Optimization in Continuous-Time Models} \label{subsec:pref_opt_continuous}
Applying~\eqref{eq:dpo} directly to the diffusion and flow-matching models introduced above is non-trivial.
Calculating the exact log-likelihood $\log \pi_\theta(\xx_0|c)$ requires solving the probability flow ODE, which is computationally prohibitive during training~\citep{zheng2025diffusionnft}.
DiffusionDPO~\cite{wallace2024diffusion} addresses this intractability by optimizing the Evidence Lower Bound (ELBO) as a proxy for the likelihood.
The loss function is reformulated using the denoising error at timestep $t$:
\begin{equation}
    \begin{aligned}
        \mathcal{L}_{\text{Diff-DPO}}(\theta) & = -\mathbb{E}_{c, \xx^w_0, \xx^l_0, t, \epsilon} \Big[ \log \sigma \Big( \beta \rho_t \Big( \\
                                                                                                     &\quad \underbrace{\left( \|\epsilon_\theta(\xx^l_t) - \epsilon\|^2 - \|\epsilon_{\text{ref}}(\xx^l_t) - \epsilon \|^2 \right)}_{\text{Advantage of Loser}}
                                                                                                     - \underbrace{\left( \|\epsilon_\theta(\xx^w_t) - \epsilon\|^2 - \|\epsilon_{\text{ref}}(\xx^w_t) - \epsilon \|^2 \right)}_{\text{Advantage of Winner}} \Big) \Big) \Big] \,,
    \end{aligned}
    \label{eq:diffusion_dpo}
\end{equation}
where $\xx_t$ is the linear interpolation of $\xx_0$ and noise $\epsilon$, and $\rho_t$ is a weighting function.
FlowDPO \cite{liu2025improving} extends this paradigm to flow matching.
By leveraging the connection between the score function and the vector field, FlowDPO replaces the noise prediction error with the flow matching loss, enabling efficient gradient-based preference optimization directly on the velocity field without ODE integration.
\looseness=-1

\subsection{Forward and Reverse KL Divergences} \label{subsec:kl_divergences}
Let $\pi(\xx_0 \mid c)$ denote a target distribution and $\pi_\theta(\xx_0 \mid c)$ denote the trainable model distribution.
Later in the paper, $\pi$ will be instantiated by the preferred and dispreferred target distributions $\pi^+$ and $\pi^-$, while the model side of the objective will be represented by branch-specific distributions $q_\theta^+$ and $q_\theta^-$ rather than by a single $\pi_\theta$.
The two KL directions are
\begin{small}
    \begin{equation}
        \label{eq:kl_directions}
        \begin{aligned}
            \text{Reverse-KL: }\mathcal{D}_{KL}(\pi_\theta \| \pi)
            = \mathbb{E}_{\xx_0 \sim \pi_\theta(\cdot \mid c)}
            \left[\log \frac{\pi_\theta(\xx_0 \mid c)}{\pi(\xx_0 \mid c)}\right] \\
            \text{Forward-KL: }\mathcal{D}_{KL}(\pi \| \pi_\theta)
            =\  \mathbb{E}_{\xx_0 \sim \pi(\cdot \mid c)}
            \left[\log \frac{\pi(\xx_0 \mid c)}{\pi_\theta(\xx_0 \mid c)}\right]
        \end{aligned}
    \end{equation}
\end{small}%
The key difference is the sampling distribution inside the expectation.
In the reverse-KL term $\mathcal{D}_{KL}(\pi_\theta \| \pi)$, errors are weighted by the current model $\pi_\theta$, so target regions that are rarely visited by $\pi_\theta$ contribute little.
In contrast, the forward-KL term $\mathcal{D}_{KL}(\pi \| \pi_\theta)$ is weighted by the target distribution $\pi$. If $\pi_\theta(\xx_0 \mid c)$ is too small in regions where $\pi(\xx_0 \mid c)$ is large, the penalty becomes severe, and finiteness requires $\operatorname{supp}(\pi) \subseteq \operatorname{supp}(\pi_\theta)$.

\section{Methodology}
\label{sec:methodology}
This section develops \method from a unified divergence-based formulation of preference alignment.
We first show that RLHF admits a contrastive reverse-KL interpretation, then derive its offline forward-KL counterpart and reduce the resulting objective to a practical flow matching loss.

\subsection{Probabilistic Reward Formulation}
\label{subsec:reward_formulation}

Let $c$ denote the conditioning context and $\xx_0$ the generated content.
Following control-as-inference~\cite{levine2018reinforcement}, we introduce a binary \emph{optimality variable} $o \in \{0, 1\}$, where $o=1$ means $\xx_0$ is preferred and $o=0$ means $\xx_0$ is dispreferred.
This gives a normalized probabilistic view of any unbounded reward $r(\xx_0, c)$:
\begin{small}
    \begin{equation}
        \begin{aligned}
            p(o = 1 \mid \xx_0, c) & = \frac{\exp(\omega \cdot r(\xx_0, c))}{\exp(\omega \cdot r(\xx_0, c)) + \exp(-\omega \cdot r(\xx_0, c))} = \sigma(2\omega \cdot r(\xx_0, c))   \\
            p(o = 0 \mid \xx_0, c) & = \frac{\exp(-\omega \cdot r(\xx_0, c))}{\exp(\omega \cdot r(\xx_0, c)) + \exp(-\omega \cdot r(\xx_0, c))} = \sigma(-2\omega \cdot r(\xx_0, c))
        \end{aligned}
        \label{eq:optimality_prob}
    \end{equation}
\end{small}%
Here $\omega > 0$ is an inverse temperature and $\sigma(x) = \frac{1}{1 + \exp(-x)}$ is the sigmoid function.
Conditioning the reference policy $\pi_{\text{ref}}(\xx_0 \mid c)$ on the two optimality events, $o=1$ and $o=0$, yields preferred and dispreferred target distributions:
\begin{equation}
    \label{eq:posteriors}
    \begin{aligned}
        \pi^+(\xx_0 \mid c) & := p(\xx_0 \mid o=1, c) = \frac{\pi_{\text{ref}}(\xx_0 \mid c)\, p(o=1 \mid \xx_0, c)}{p_{\pi_{\text{ref}}}(o=1 \mid c)} \\
        \pi^-(\xx_0 \mid c) & := p(\xx_0 \mid o=0, c) = \frac{\pi_{\text{ref}}(\xx_0 \mid c)\, p(o=0 \mid \xx_0, c)}{p_{\pi_{\text{ref}}}(o=0 \mid c)}
    \end{aligned}
\end{equation}
We use $\pi^+$ and $\pi^-$ as the positive and negative target distributions throughout the paper.
In practice, offline preference construction may only produce empirical approximations to these posteriors. We defer this distinction to \appref{app:source_distributions}.

By inverting~\eqref{eq:posteriors} and substituting into~\eqref{eq:optimality_prob}, the reward $r(\xx_0, c)$ becomes
\begin{equation}
    \label{eq:reward}
    r(\xx_0, c)
    = \frac{1}{2\omega} \log\frac{p(o = 1 \mid \xx_0, c)}{p(o = 0 \mid \xx_0, c)}
    = \frac{1}{2\omega} \left[\log\frac{\pi^+(\xx_0 \mid c)}{\pi^-(\xx_0 \mid c)} + \log\frac{p_{\pi_{\text{ref}}}(o=1\mid c)}{p_{\pi_{\text{ref}}}(o=0\mid c)}\right] \,.
\end{equation}
Thus, up to a context-only constant, reward is a contrastive log-density ratio: high reward corresponds to regions where $\pi^+$ dominates $\pi^-$.

\subsection{Re-formalizing Reinforcement Learning from Human Feedback}

We begin with the unregularized expected-reward term $\E_{c, \xx_0 \sim \pi_\theta(\cdot \mid c)} [r(\xx_0, c)]$ that appears inside RLHF objectives.
Using~\eqref{eq:reward}, we can rewrite this term directly in divergence form in \thmref{thm:rlhf_kl} (see \appref{app:rlhf_kl} for the proof):
\begin{theorem}[Expected reward as contrastive reverse KL]
    \label{thm:rlhf_kl}
    Up to the positive scale factor $\frac{1}{2\omega}$ and an additive term depending only on $c$, maximizing expected reward is equivalent to minimizing
    \begin{equation}
        \label{eq:rlhf_kl}
        \E_c \left[ \mathcal{D}_{KL}(\pi_\theta \| \pi^+) - \mathcal{D}_{KL}(\pi_\theta \| \pi^-) \right] \,.
    \end{equation}
\end{theorem}
\Eqref{eq:rlhf_kl} shows that expected reward induces a contrastive reverse-KL term: the policy is attracted to $\pi^+$ and repelled from $\pi^-$ under its own sampling distribution.
Reference-policy KL penalties used in practical RLHF methods remain additional terms and are not absorbed by this identity.
Due to the high computational cost of online sampling, we focus on the offline regime in this work.
To derive an offline counterpart, we reverse the direction of the KL divergence in \eqref{eq:rlhf_kl} and assume fixed sampling distributions:
\begin{equation}
    \label{eq:op_kl}
    \min_\theta \mathcal{L}_{\text{Offline}}
    = \E_c \left[ \mathcal{D}_{KL}(\pi^+ \| \pi_\theta^+) + \lambda \cdot \mathcal{D}_{KL}(\pi^- \| \pi_\theta^-) \right] \,,
    \qquad \lambda \ge 0 \,.
\end{equation}
Here $\pi_\theta^+$ and $\pi_\theta^-$ are branch-specific model distributions.
For \method itself, we focus on the attractive regime $\lambda \ge 0$. In \secref{subsec:negative_regularization}, we additionally instantiate $\lambda < 0$ under the generalized framework to analyze repulsive baselines.

To avoid evaluating flow-model likelihoods directly, we parameterize the two branches through symmetric mixed velocity fields:
\begin{align*}
    \mu_\theta = (1-\beta)\, v_{\text{old}} + \beta\, v_\theta \,, \qquad
    \nu_\theta = (1+\beta)\, v_{\text{old}} - \beta\, v_\theta \,.
\end{align*}
Here $v_{\text{old}}$ is an EMA copy of the trainable field, and $\beta>0$ controls the symmetric displacement of the two branches from that reference.
The implicit distributions induced from the common Gaussian endpoint by $\mu_\theta$ and $\nu_\theta$ are denoted by $\pi_\theta^+$ and $\pi_\theta^-$. The idealized score-field interpretation of this interpolation is given in \appref{app:interpolation_vector_field}.

The remaining bridge is from cross-entropy to regression.
Each forward-KL term equals a target entropy plus an expected model negative log-likelihood; under the uniform regularity conditions in \appref{app:relation_loglike_fm}, the latter is bounded by conditional flow matching objective plus a constant independent of $\theta$.
Applying this argument separately to the two branches yields \thmref{thm:flowcpo_bound}.
\begin{theorem}[A flow-matching upper bound for offline forward KL]
    \label{thm:flowcpo_bound}
    Under the regularity conditions in \appref{app:relation_loglike_fm}, we have:
    \begin{equation}
        \label{eq:weighted_flowcpo}
        \begin{aligned}
            \mathcal{L}_{\text{Offline}}(\theta)
            \le C
             & + \E_{\substack{c,\, \xx_0^+ \sim \pi^+(\cdot \mid c) \\ t,\, \epsilon}}
            \Big[
                 \big\| \mu_\theta(\xx_t^+, t, c) - u_t(\xx_t^+ \mid \xx_0^+) \big\|_2^2
            \Big] \\
             & + \lambda \cdot
            \E_{\substack{c,\, \xx_0^- \sim \pi^-(\cdot \mid c) \\ t,\, \epsilon}}
            \Big[
                \big\| \nu_\theta(\xx_t^-, t, c) - u_t(\xx_t^- \mid \xx_0^-) \big\|_2^2
            \Big] \,,
        \end{aligned}
    \end{equation}
    where $C$ is independent of $\theta$.
\end{theorem}
The proof of \thmref{thm:flowcpo_bound} is deferred to \appref{app:fm_bound}.
This result justifies a tractable upper-bound surrogate.
In practice, we replace the idealized source distributions $\pi^+, \pi^-$ with empirical preferred and dispreferred datasets, $\mathcal{D}^+$ and $\mathcal{D}^-$, and minimize the following contrastive flow matching loss:
\begin{tcolorbox}[colback=eqbg, colframe=white, boxrule=0pt]
    \begin{equation}
        \label{eq:final_loss}
        \begin{aligned}
            \min_\theta \mathcal{L}_{\text{\method}}(\theta) =
             & \qquad \E_{\substack{(c, \xx_0^w) \sim \mathcal{D}^+ \\ t, \epsilon}}
            \Big[ \big\| \mu_{\theta}(\xx_t^w, t, c) - u_t(\xx_t^w \mid \xx_0^w) \big\|_2^2 \Big] \\
             & \qquad + \, \lambda \cdot \,
            \E_{\substack{(c, \xx_0^l) \sim \mathcal{D}^- \\ t, \epsilon}}
            \Big[ \big\| \nu_{\theta}(\xx_t^l, t, c) - u_t(\xx_t^l \mid \xx_0^l) \big\|_2^2 \Big].
        \end{aligned}
    \end{equation}
\end{tcolorbox}
Pseudo-code is given in \algmref{alg:flowcpo}.
For $\lambda\geq0$, \eqref{eq:final_loss} is bounded below by zero, whereas simplified FlowDPO's signed objective can be unbounded below. \appref{app:flowdpo_ema_connection} provides the comparison.

\subsection{Unified View and Connections to Prior Work}
\label{subsec:connections}
\tabref{tab:comparison} summarizes connections between \eqref{eq:generalized_obj} and prior works, and \appref{app:connections} elaborates the assumptions behind each mapping.
\method is derived from the forward-KL objective in \eqref{eq:op_kl}, which is a special case of \eqref{eq:generalized_obj}.
The generalized objective organizes flow-based preference-alignment methods (FlowDPO, FlowGRPO, DiffusionNFT) and regression baselines (SFT, RFT):
\begin{tcolorbox}[colback=eqbg, colframe=white, boxrule=0pt]
    \begin{equation}
        \label{eq:generalized_obj}
        \min_\theta \mathcal{L}_{\text{General}}
        = \E_c \left[
            \alpha \cdot \underbrace{\mathcal{D}_f(\pi^+ \| q_\theta^+)}_{\text{positive alignment}}
            \, + \, 
            \gamma \cdot \underbrace{\mathcal{D}_f(\pi^- \| q_\theta^-)}_{\text{negative regularization}}
            \right]
    \end{equation}
\end{tcolorbox}
Here $\alpha > 0$, and $\gamma \in \mathbb{R}$.
In the forward-KL branch we use $\mathcal{D}_f(\pi^\pm \| q)=\mathcal{D}_{KL}(\pi^\pm \| q)$, and in the reverse-KL branch we use $\mathcal{D}_f(\pi^\pm \| q)=\mathcal{D}_{KL}(q \| \pi^\pm)$.
When $\gamma \ge 0$, the negative regularization matches a negative distribution, and it acts as repulsive regularization when $\gamma < 0$.

\begin{table}[!t]
    \centering
    \caption{
        \textbf{Comparison of flow-based alignment methods and related special cases under~\eqref{eq:generalized_obj}.}
        Rows are organized by optimization regime and divergence choice.
        BT indicates Bradley--Terry preference model, and Ref.\ KL indicates whether the method uses reference-model KL regularization.
    }
    \label{tab:comparison}
    \resizebox{\textwidth}{!}{%
        \begin{tabular}{lccccccccc}
            \toprule
            \textbf{Method}                                 & \textbf{Regime}  & \textbf{Divergence} & \textbf{Optimized Process} & $\bm{\alpha}$ & $\bm{\gamma}$  & $q_\theta^+$        & $q_\theta^-$        & \textbf{BT}  & \textbf{Ref.\ KL} \\
            \midrule
            FlowGRPO~\cite{liu2025flow}                     & Online           & Reverse KL          & Reverse Process & $1$           & $-1$           & $\pi_\theta$        & $\pi_\theta$        & $\times$     & $\checkmark$      \\
            DiffusionNFT~\cite{zheng2025diffusionnft}       & Online           & Forward KL          & Forward Process & $p(o=1|c)$    & $p(o=0|c)$     & $\pi_\theta^+$      & $\pi_\theta^-$      & $\times$     & $\times$          \\
            AWM~\cite{xue2025advantage}                     & Online           & Reverse KL          & Forward Process & $1$    & $-1$     & $\pi_\theta$      & $\pi_\theta$      & $\times$     & $\times$          \\
            RAM~\cite{bergmeister2026reinforce}             & Online           & Reverse KL          & Forward Process & $1$    & $-1$     & $\pi_\theta$      & $\pi_\theta$      & $\times$     & $\checkmark$          \\
            \midrule
            SFT                                             & Offline          & Forward KL          & - & $p(o=1|c)$    & $p(o=0|c)$     & $\pi_\theta$        & $\pi_\theta$        & $\times$     & $\times$          \\
            RFT~\cite{xiong2025minimalist,chen2025bridging} & Offline          & Forward KL          & - & $1$           & $0$            & $\pi_\theta$        & $-$                 & $\times$     & $\times$          \\
            FlowDPO~\cite{liu2025improving}                 & Offline          & Implicit reverse KL & - &  $1$           & $-1$           & $\pi_\theta$        & $\pi_\theta$        & $\checkmark$ & $\checkmark$      \\
            \midrule
            \textbf{\method}                                & \textbf{Offline} & \textbf{Forward KL} & - & $\bm{1}$      & $\bm{\lambda}$ & $\bm{\pi_\theta^+}$ & $\bm{\pi_\theta^-}$ & $\times$     & $\bm{\times}$     \\
            \bottomrule
        \end{tabular}
    }
\end{table}

\section{Experiments}
\label{sec:experiments}

We study \method along three axes: offline alignment with in-domain and out-of-domain preference data, the effect of negative regularization, and sensitivity to the core hyperparameters $(\beta, \lambda, \eta)$.
The main text focuses on the central quantitative and qualitative results, while the appendix collects the training pseudo-code, per-domain tables, and extended ablations.

\subsection{Experimental Setup}
\label{subsec:exp_setup}

\begin{table*}[!t]
    \centering
    \caption{\small
        \textbf{Quantitative comparison under \emph{in-domain} and \emph{out-of-domain} (OOD) offline training regimes} on \texttt{SD3.5-M}.
        We report results with and without Classifier-Free Guidance (CFG).
        When CFG is enabled, we report separate rows for guidance scales $\{1.0, 3.0, 4.5\}$.
        Under the \emph{in-domain} setting, fine-tuned methods are trained separately for each target domain (Concept, Typography, and General Preference).
        Under the \emph{OOD} setting, a single model is trained on the OOD offline dataset and evaluated across all domains and metrics.
        The two blocks therefore correspond to different evaluation protocols and should be compared primarily within, rather than across, regimes.
        The main table reports means only and rounds all entries to 2-3 decimal places; the corresponding mean $\pm$ standard deviation statistics over 5 independent runs for fine-tuned \texttt{SD3.5-M} variants are reported in the \appref{app:optimization_targets}, whereas pretrained baselines are single evaluations. Within each training regime, the best result is highlighted in \textbf{bold} and the second-best result is \underline{underlined}.
    }
    \label{tab:main_results}
    \resizebox{1\textwidth}{!}{
        \begin{tabular}{l c c c c c c c c c}
            \toprule
                           &              & \multicolumn{1}{c}{\textbf{Concept}} & \multicolumn{1}{c}{\textbf{Typography}} & \multicolumn{6}{c}{\textbf{General Preference}}                                                                                                                                                                           \\
            \cmidrule(lr){3-3} \cmidrule(lr){4-4} \cmidrule(lr){5-10}
            \textbf{Model} & \textbf{CFG} & \textbf{GenEval $\uparrow$}          & \textbf{OCR $\uparrow$}                 & \textbf{PickScore $\uparrow$}                   & \textbf{CLIPSc. $\uparrow$}     & \textbf{HPSv2.1 $\uparrow$}     & \textbf{Aes. $\uparrow$}        & \textbf{ImgRwd $\uparrow$}      & \textbf{UniRwd $\uparrow$}      \\
            \midrule

            \multicolumn{10}{l}{\textit{Pretrained model baselines}}                                                                                                                                                                                                                                                                                  \\
            SD-XL          & $-$          & 0.55                                 & 0.14                                    & 22.42                                           & 0.287                           & 0.280                           & 5.60                            & 0.76                            & 2.93                            \\
            SD3.5-L        & $-$          & 0.71                                 & 0.68                                    & 22.91                                           & 0.289                           & 0.288                           & 5.50                            & 0.96                            & 3.25                            \\
            FLUX.1-Dev     & $-$          & 0.66                                 & 0.59                                    & 22.84                                           & 0.295                           & 0.274                           & 5.71                            & 0.96                            & 3.27                            \\
            \midrule

            \multicolumn{10}{l}{\textit{Reference model: SD3.5-M}}                                                                                                                                                                                                                                                                                     \\
            \multirow{3}{*}{Base Model}
                           & $1.0$        & 0.24                                 & 0.12                                    & 20.51                                           & 0.237                           & 0.204                           & 5.13                            & $-$0.58                         & 2.02                            \\
                           & $3.0$        & 0.59                                 & 0.47                                    & 22.28                                           & 0.287                           & 0.284                           & 5.38                            & 0.71                            & 2.96                            \\
                           & $4.5$        & 0.63                                 & 0.59                           & 22.34                                           & 0.285                           & 0.279                           & 5.36                            & 0.85                            & 3.03                            \\
            \midrule

            \multicolumn{10}{>{\columncolor{indomaincolor}}l}{\textit{In-domain offline fine-tuning on SD3.5-M}}                                                                                                                                                                                                                                       \\
            \multirow{3}{*}{+ RFT~\cite{xiong2025minimalist,chen2025bridging} }
                           & $1.0$        & $0.59$                  & $0.35$                     & $21.91$                            & $0.279$             & $0.276$             & $5.35$             & $0.57$             & $2.66$             \\
                           & $3.0$        & $0.74$                  & $0.70$                     & $22.60$                            & $0.296$             & $0.304$             & $5.41$             & $1.06$             & $3.14$             \\
                           & $4.5$        & $0.75$                  & $0.72$                     & $22.57$                            & $0.297$             & \underline{$0.305$}             & $5.42$             & $1.11$             & $3.15$             \\
            \cmidrule(l){2-10}
            \multirow{3}{*}{+ FlowDPO~\cite{liu2025improving}}
                           & $1.0$        & $0.59$                  & $0.51$                     & $20.72$                            & $0.240$             & $0.221$             & $5.22$             & $-0.46$            & $2.11$             \\
                           & $3.0$        & $0.81$                  & $0.74$                     & $22.76$                            & $0.297$             & $0.301$             & \underline{$5.56$} & $1.04$             & $3.10$             \\
                           & $4.5$        & $0.81$                  & $0.75$                     & \underline{$22.89$}                & $\mathbf{0.301}$    & $\mathbf{0.311}$ & \underline{$5.56$}             & $1.15$             & $3.18$             \\
            \cmidrule(l){2-10}
            \multirow{3}{*}{+ \method ($\beta = 0.5$, Ours)}
                           & $1.0$        & $0.76$                  & $0.83$                     & $22.48$                            & $0.280$             & $0.290$             & $\mathbf{5.64}$    & $0.90$             & $2.93$             \\
                           & $3.0$        & $\mathbf{0.84}$         & $\mathbf{0.87}$            & $\mathbf{22.94}$                   & $0.299$             & $\mathbf{0.311}$    & $5.54$             & $\mathbf{1.25}$    & $\mathbf{3.30}$    \\
                           & $4.5$        & \underline{$0.82$}      & \underline{$0.86$}         & $22.67$                            & \underline{$0.300$} & $0.303$             & $5.46$             & \underline{$1.22$} & \underline{$3.28$} \\
            \midrule
            \multicolumn{10}{>{\columncolor{oodcolor}}l}{\textit{Out-of-domain (OOD) offline fine-tuning on SD3.5-M}}                                                                                                                                                                                                                                  \\
            \multirow{3}{*}{+ RFT~\cite{xiong2025minimalist,chen2025bridging}}
                           & $1.0$        & $0.44$                  & $0.17$                     & $21.59$                            & $0.268$             & $0.261$             & $5.49$             & $0.31$             & $2.56$             \\
                           & $3.0$        & $0.67$                  & $0.53$                     & \underline{$22.64$}                & \underline{$0.297$} & \underline{$0.303$} & $5.48$             & \underline{$1.05$} & $3.18$             \\
                           & $4.5$        & $0.68$                  & $\mathbf{0.60}$            & $\mathbf{22.69}$                   & $\mathbf{0.299}$    & $\mathbf{0.307}$    & $5.47$             & $\mathbf{1.12}$    & $\mathbf{3.23}$    \\
            \cmidrule(l){2-10}
            \multirow{3}{*}{+ FlowDPO~\cite{liu2025improving}}
                           & $1.0$        & $0.23$                  & $0.12$                     & $20.83$                            & $0.243$             & $0.226$             & $\mathbf{5.72}$    & $-0.39$            & $2.27$             \\
                           & $3.0$        & $0.61$                  & $0.48$                     & $22.54$                            & $0.292$             & $0.292$             & \underline{$5.52$} & $0.91$             & $3.10$             \\
                           & $4.5$        & $0.65$                  & $0.54$                     & $22.62$                            & $0.295$             & $0.299$             & $5.49$             & $0.99$             & $3.17$             \\
            \cmidrule(l){2-10}
            \multirow{3}{*}{+ \method ($\beta=0.5$, Ours)}
                           & $1.0$        & $0.57$                  & $0.24$                     & $22.02$                            & $0.281$             & $0.275$             & $5.35$             & $0.63$             & $2.85$             \\
                           & $3.0$        & \underline{$0.69$}                  & $0.49$                     & $22.32$                            & $0.296$             & $0.288$             & $5.37$             & $1.00$             & $3.15$             \\
                           & $4.5$        & $0.68$                  & $0.49$                     & $22.12$                            & $0.294$             & $0.284$             & $5.31$             & $0.95$             & $3.11$             \\
            \cmidrule(l){2-10}
            \multirow{3}{*}{+ \method ($\beta=1$, Ours)}
                           & $1.0$        & $0.47$                  & $0.24$                     & $21.81$                            & $0.276$             & $0.270$             & $5.35$             & $0.52$             & $2.77$             \\
                           & $3.0$        & $\mathbf{0.70}$      & $0.56$                     & $22.46$                            & \underline{$0.297$} & $0.295$             & $5.41$             & $1.02$             & $\mathbf{3.23}$ \\
                           & $4.5$        & $\mathbf{0.70}$         & \underline{$0.59$}         & $22.36$                            & $0.294$             & $0.294$             & $5.38$             & $1.02$             & \underline{$3.22$}            \\
            \bottomrule
        \end{tabular}
    }
\end{table*}

\paragraph{Data Construction and Implementation.}
We evaluate \method on \texttt{Stable Diffusion 3.5 Medium (SD3.5-M)}~\cite{esser2024scaling} in a strictly offline setting.
Our offline preference data come from two sources. In the \emph{in-domain} regime, we build winner/loser pairs from a frozen copy of the reference model using prompts drawn from the target benchmark sources. In the \emph{out-of-domain} (OOD) regime, we use the public Open Image Preferences v1 dataset, whose images are generated by other open models. All models are trained under the same offline protocol, and we defer the exact data-generation pipeline, dataset sizes, reward filtering, and training hyperparameters to \appref{app:experimental_details}. In the in-domain setting, each target capability uses its own specialist model, so the \tabref{tab:main_results} summarizes task-specific fine-tuning results rather than a single universal model.

\paragraph{Baselines.}
We compare against the two published approaches that are most directly comparable under the same fixed-data protocol:
\begin{itemize} [leftmargin=*, itemsep=0pt, topsep=2pt]
    \item \textbf{RFT (Rejection Sampling Fine-Tuning)}~\cite{xiong2025minimalist,chen2025bridging}: A positive-only offline baseline that can be viewed as a forward-KL special case, learning exclusively from preferred data and ignoring dispreferred data.
    \item \textbf{FlowDPO}~\cite{liu2025improving}: A flow-based DPO baseline that uses preferred/dispreferred pairs through a reference-relative likelihood-ratio surrogate.
\end{itemize}

\begin{figure}[t]
    \centering
    \includegraphics[width=1\textwidth]{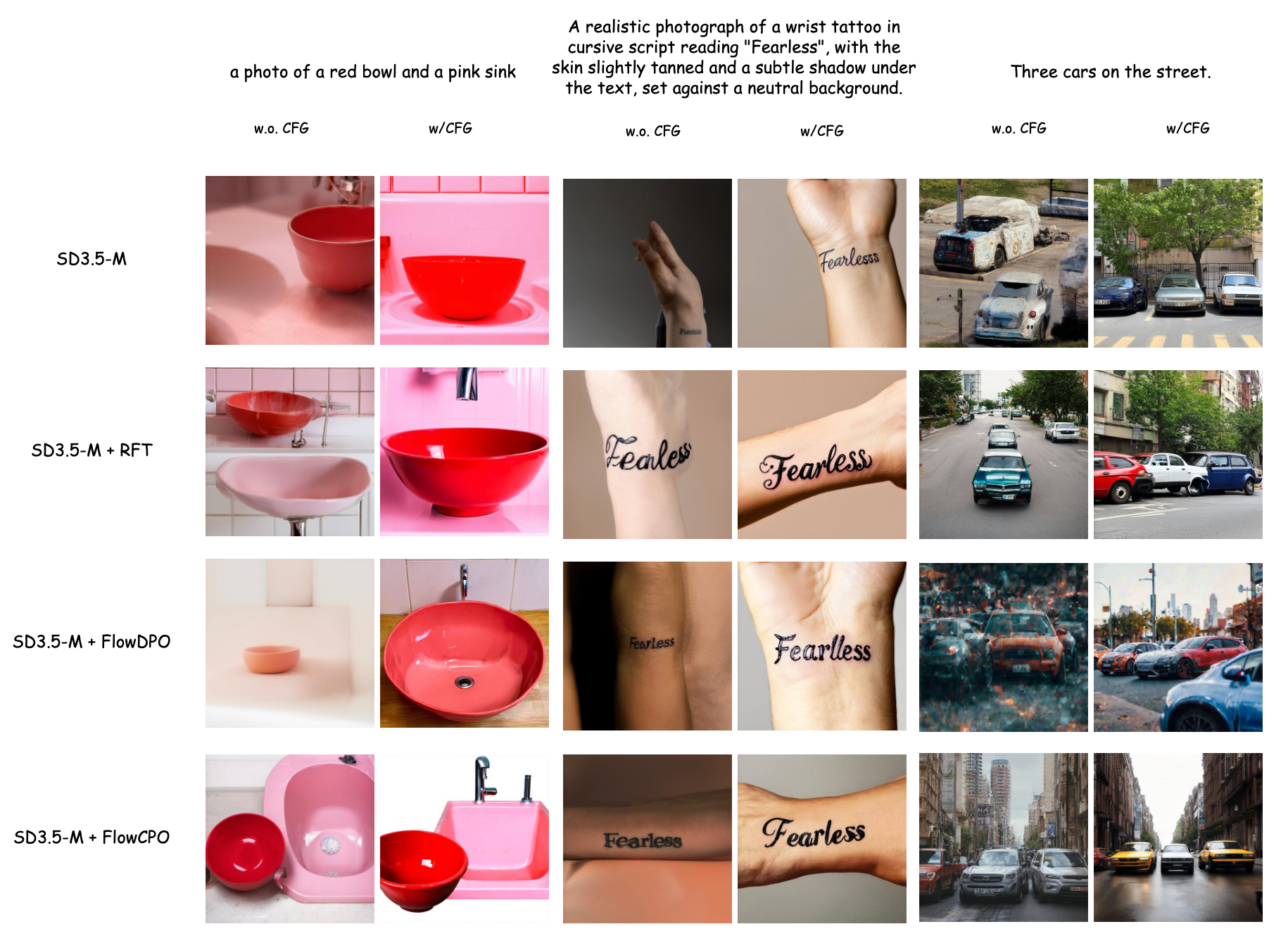}
    \caption{\small
        \textbf{Representative qualitative comparison of \method and the baselines} across three preference optimization tasks.
    }
    \vspace{-1.5em}
    \label{fig:all_flowcpo_image_quality}
\end{figure}

\paragraph{Evaluation Metrics.}
We evaluate two kinds of behavior. For targeted capabilities, we use task-specific benchmarks: \texttt{GenEval}~\cite{ghosh2023geneval} for multi-concept compositional generation and an \texttt{OCR}-based metric for visual text rendering. For general preference alignment, we run inference on the \texttt{DrawBench} prompt set and report \texttt{PickScore}~\cite{kirstain2023pick}, \texttt{CLIP Score}~\cite{hessel2021clipscore}, \texttt{HPS v2.1}~\cite{wu2023hpsv2}, \texttt{Aesthetics}~\cite{schuhmann2022laion}, \texttt{ImageReward}~\cite{xu2023imagereward}, and \texttt{UnifiedReward}~\cite{wang2025unified}. Because the in-domain general-preference subset is filtered with \texttt{PickScore}, \texttt{CLIP Score}, and \texttt{HPS v2.1}, we treat these as optimization-aligned metrics and use \texttt{Aesthetics}, \texttt{ImageReward}, and \texttt{UnifiedReward} as a cleaner check of transfer beyond the filtering pipeline.

\subsection{RQ1: Main Results with In-Domain and Out-of-Domain Data}
\label{subsec:main_results}

To evaluate method-level performance, we study three capability groups: semantic alignment, typographic generation, and general preference.

We summarize the main quantitative results in \tabref{tab:main_results}. The two regimes answer different questions, so we read them separately: the \emph{in-domain} block evaluates in-domain offline fine-tuning, whereas the \emph{out-of-domain} (OOD) block evaluates a single model trained once on an OOD offline dataset.
\begin{itemize}[leftmargin=*, itemsep=0pt, topsep=2pt]
    \item \textbf{In the \emph{in-domain} setting, \method has higher mean GenEval and OCR scores than RFT and FlowDPO.} \tabref{tab:main_results} aggregates separate task-specific fine-tunes for the three capability groups; \method attains the best GenEval mean (\textbf{0.84} at CFG 3.0) and the best OCR mean (\textbf{0.87} at CFG 3.0), exceeding FlowDPO by 0.03 and 0.12, respectively. On general preference, the picture is narrower but still favorable: \method is best or tied-best on \texttt{PickScore}, \texttt{HPS v2.1}, \texttt{ImageReward}, and \texttt{UnifiedReward}, while \texttt{CLIP Score} is effectively tied with FlowDPO. Since \texttt{PickScore}, \texttt{CLIP Score}, and \texttt{HPS v2.1} also appear in the in-domain filtering pipeline, we put more weight on \texttt{Aesthetics}, \texttt{ImageReward}, and \texttt{UnifiedReward}, together with the qualitative comparison in \figref{fig:all_flowcpo_image_quality} and the extended examples in \appref{app:optimization_targets}, when judging transfer beyond the optimization-aligned signals.
    \item \textbf{In the \emph{out-of-domain} setting, the advantage is metric-dependent.} \method gives the best GenEval score and remains competitive on OCR and \texttt{UnifiedReward}, whereas RFT performs better on several reward-model metrics. Thus, retaining both sides of each preference pair does not guarantee improvement when the fixed data are generated by models other than the reference model. This result motivates the regularization analysis in RQ2 and identifies out-of-domain data as an empirical boundary of the method.
\end{itemize}

\begin{table*}[!t]
    \centering
    \caption{\small
        \textbf{Analysis of different regularization mechanisms} on \texttt{SD3.5-M} under the unified divergence-based framework of \eqref{eq:generalized_obj}, with all models trained only on \emph{GenEval}.
        We consider three representative paradigms: positive-only regularization, which uses only preferred samples; repulsive negative regularization, which explicitly pushes the model away from dispreferred regions; and attractive negative regularization, which incorporates negative samples through distributional matching.
        For attractive negative regularization, we report two instantiations with different mixing coefficients $\beta$.
        Results are reported with and without Classifier-Free Guidance (CFG).
        Metrics are grouped into \emph{Concept}, \emph{Typography}, and \emph{General Preference} for consistency with the main results table and to assess cross-metric generalization beyond the training signal. Best results are highlighted in \textbf{bold} and second-best results are \underline{underlined}.
    }
    \vspace{-0.5em}
    \label{tab:regularization_mechanisms}
    \resizebox{1\textwidth}{!}{
        \begin{tabular}{c c c c c c c c c c}
            \toprule
            \multirow{2}{*}{\textbf{$\lambda$}} & \multirow{2}{*}{\textbf{CFG}}
                                                & \multicolumn{1}{c}{\textbf{Concept}}
                                                & \multicolumn{1}{c}{\textbf{Typography}}
                                                & \multicolumn{6}{c}{\textbf{General Preference}}                                                                                                                                                                                                                                          \\
            \cmidrule(lr){3-3} \cmidrule(lr){4-4} \cmidrule(lr){5-10}
                                                &                                                 & \textbf{GenEval $\uparrow$} & \textbf{OCR $\uparrow$} & \textbf{PickScore $\uparrow$} & \textbf{CLIPSc. $\uparrow$} & \textbf{HPSv2.1 $\uparrow$} & \textbf{Aes. $\uparrow$} & \textbf{ImgRwd $\uparrow$} & \textbf{UniRwd $\uparrow$} \\
            \midrule

            \rowcolor{gray!15}\multicolumn{10}{l}{\textit{Positive-only regularization (RFT-like):} $\mu_\theta=v_\theta,\ \nu_\theta=0$}                                                                                                                                                                                                  \\

            $0$                                 & $ 1.0 $                                         & 0.59                        & 0.15                    & 21.67                         & 0.27                        & 0.268                       & 5.32                     & 0.43                       & 2.63                       \\
            $0$                                 & $ 3.0 $                                         & 0.72                        & 0.51                    & \textbf{22.5}                 & 0.295                       & \underline{0.301}           & \textbf{5.41}            & 1.03                       & 3.13                       \\
            $0$                                 & $ 4.5 $                                         & 0.75                        & 0.56                    & \textbf{22.5}                 & 0.296                       & \textbf{0.304}              & \textbf{5.41}            & 1.08                       & 3.16                       \\
            \midrule

            \rowcolor{gray!15}\multicolumn{10}{l}{\textit{Repulsive negative regularization (DPO-like):} $\mu_\theta=\nu_\theta=v_\theta$}                                                                                                                                                                                                 \\
            $-0.1$                              & $ 1.0 $                                         & 0.64                        & 0.15                    & 21.76                         & 0.277                       & 0.270                       & 5.30                     & 0.53                       & 2.66                       \\
            $-0.1$                              & $ 3.0 $                                         & 0.77                        & 0.50                    & \underline{22.48}             & 0.297                       & 0.300                       & 5.37                     & 1.07                       & 3.17                       \\
            $-0.1$                              & $ 4.5 $                                         & 0.76                        & \underline{0.58}        & 22.47                         & 0.297                       & \textbf{0.304}              & \underline{5.40}         & \underline{1.10}           & 3.17                       \\
            $-1$                                & $ 1.0 $                                         & 0.25                        & 0.08                    & 19.46                         & 0.201                       & 0.154                       & 4.22                     & -1.66                      & 1.40                       \\
            $-1$                                & $ 3.0 $                                         & 0.70                        & 0.48                    & 21.39                         & 0.270                       & 0.229                       & 4.70                     & -0.28                      & 2.51                       \\
            $-1$                                & $ 4.5 $                                         & 0.71                        & 0.55                    & 21.50                         & 0.276                       & 0.240                       & 4.86                     & -0.01                      & 2.65                       \\
            $-10$                               & --                                              & Diverged                    & Diverged                & Diverged                      & Diverged                    & Diverged                    & Diverged                 & Diverged                   & Diverged                   \\
            \midrule

            \rowcolor{gray!15}\multicolumn{10}{l}{\textit{Attractive negative regularization (CPO-like):} $\mu_\theta=v_\theta,\ \nu_\theta= 2 v_{\text{old}} - v_\theta$}                                                                                                                                                                 \\
            $0.1$                               & $ 1.0 $                                         & 0.64                        & 0.17                    & 21.72                         & 0.278                       & 0.266                       & 5.24                     & 0.50                       & 2.71                       \\
            $0.1$                               & $ 3.0 $                                         & 0.77                        & 0.54                    & 22.47                         & 0.297                       & \underline{0.301}           & 5.35                     & 1.06                       & \underline{3.19}           \\
            $0.1$                               & $ 4.5 $                                         & 0.77                        & \textbf{0.59}           & 22.48                         & \underline{0.298}           & \textbf{0.304}              & 5.38                     & \textbf{1.14}              & \textbf{3.23}              \\
            $1$                                 & $ 1.0 $                                         & \underline{0.78}            & 0.25                    & 21.74                         & 0.280                       & 0.258                       & 5.23                     & 0.57                       & 2.74                       \\
            $1$                                 & $ 3.0 $                                         & \underline{0.78}            & 0.51                    & 21.69                         & 0.291                       & 0.266                       & 5.09                     & 0.76                       & 2.96                       \\
            $1$                                 & $ 4.5 $                                         & 0.67                        & 0.55                    & 21.20                         & 0.285                       & 0.245                       & 4.89                     & 0.43                       & 2.75                       \\
            $10$                                & --                                              & Diverged                    & Diverged                & Diverged                      & Diverged                    & Diverged                    & Diverged                 & Diverged                   & Diverged                   \\
            \midrule

            \rowcolor{gray!15}\multicolumn{10}{l}{\textit{Attractive negative regularization (CPO-like):} $\mu_\theta=0.5 v_{\text{old}} + 0.5 v_\theta,\ \nu_\theta= 1.5 v_{\text{old}} - 0.5 v_\theta$}                                                                                                                                  \\
            $0.1$                               & $ 1.0 $                                         & 0.63                        & 0.17                    & 21.48                         & 0.275                       & 0.259                       & 5.22                     & 0.33                       & 2.55                       \\
            $0.1$                               & $ 3.0 $                                         & \underline{0.78}            & 0.51                    & 22.41                         & \underline{0.298}           & 0.297                       & 5.33                     & 1.04                       & 3.14                       \\
            $0.1$                               & $ 4.5 $                                         & 0.77                        & \underline{0.58}        & 22.43                         & \textbf{0.299}              & 0.299                       & 5.35                     & \underline{1.10}           & \underline{3.19}           \\
            $1$                                 & $ 1.0 $                                         & 0.76                        & 0.26                    & 21.84                         & 0.280                       & 0.268                       & 5.26                     & 0.64                       & 2.79                       \\
            $1$                                 & $ 3.0 $                                         & \textbf{0.84}               & 0.53                    & 22.20                         & 0.296                       & 0.288                       & 5.27                     & 1.01                       & 3.10                       \\
            $1$                                 & $ 4.5 $                                         & 0.83                        & 0.55                    & 21.94                         & 0.294                       & 0.282                       & 5.16                     & 0.95                       & 3.04                       \\
            $10$                                & --                                              & Diverged                    & Diverged                & Diverged                      & Diverged                    & Diverged                    & Diverged                 & Diverged                   & Diverged                   \\
            \bottomrule
        \end{tabular}
    }
    \vspace{-1em}
\end{table*}

\subsection{RQ2: Effectiveness of Different Negative Regularization}
\label{subsec:negative_regularization}

We compare positive-only training with repulsive ($\lambda<0$) and attractive ($\lambda>0$) negative regularization under the unified framework.

\tabref{tab:regularization_mechanisms} \textbf{favors moderate attractive matching}: $\beta=0.5,\lambda=1$ achieves the best GenEval score (0.84), while $\beta=1,\lambda=0.1$ gives the best OCR, \texttt{ImageReward}, and \texttt{UnifiedReward}. Repulsion with $\lambda=-1$ sharply degrades all metrics, and both signs diverge at $|\lambda|=10$. Since training optimizes only GenEval, the other gains measure cross-metric transfer.

\subsection{RQ3: Sensitivity of \method to Its Core Hyperparameters}
\label{subsec:hyperparameter_sensitivity}

\figref{fig:ablation_studies} favors moderate interpolation ($\beta\in[0.5,1.0]$), balanced negative weight ($\lambda=1$), and slow EMA ($\eta=0.99$). Extreme weights destabilize training, while small weights weaken the negative contribution. We therefore use $\beta=0.5$, $\lambda=1$, and $\eta=0.99$ as the default settings. Complete sweeps appear in \appref{app:ablation}.

\section{Conclusion and Limitations}
\label{sec:conclusion}

We presented a divergence-based framework that separates sampling regime from divergence direction and derived \method as its offline forward-KL instance. Coupled preferred and dispreferred branches enable regression on fixed preference pairs without online rollouts. Our loss decomposition also connects \method to simplified FlowDPO and shows how the difference between the current model and its EMA enters the objective. The nonnegative loss avoids simplified FlowDPO's potentially unbounded negative regression term.

Experiments show higher mean in-domain GenEval and OCR scores than RFT and FlowDPO, with mixed out-of-domain gains. The ablations favor moderate attractive matching of dispreferred samples, while excessively large weights can destabilize training.

The forward-KL bound assumes linear interpolation and uniform regularity. Loss nonnegativity alone does not guarantee stable training. We evaluate fine-tuning on SD3.5-M, so performance on other backbones remains to be established. Extending the bound to other paths and improving robustness across different data sources remain open challenges.

\bibliography{resources/reference}
\bibliographystyle{plainnat}

\appendix

\newpage

\begingroup
\setlength{\parskip}{0pt}
\hypersetup{linkcolor=black}
\tableofcontents
\endgroup

\newpage

\section{Theoretical Relation between Log-Likelihood and Flow Matching Loss}
\label{app:relation_loglike_fm}
In this section, we discuss the relationship between log-likelihood and flow matching loss. The theoretical results mainly follow from \citep{song2021maximum,han2026conditional,lu2022maximumode}.

\subsection{Pointwise Relation between Log-Likelihood and Flow Matching Loss}
The theoretical results in this subsection are from \citep{han2026conditional}. 
Fix a sample $x_0\in\R^d$, 
\[
\xx_t=(1-t)x_0+t\xx_1 \sim q_t^{x_0}=\cN((1-t)x_0,t^2I_d), \qquad \xx_1\sim\cN(0,I_d).
\]
Let $u_t^{x_0}(\xx)=\frac{\xx-x_0}{t}$ be the velocity field corresponding to $q_t^{x_0}$.  We compare this conditional path $(q_t^{x_0},u_t^{x_0})$ to the model path $(p_t^\theta,v_\theta)$ through the velocity and score gaps:
\[
\Delta v_t^\theta(\xx;x_0)
:=
v_\theta(\xx,t)-u_t^{x_0}(\xx),
\qquad
\Delta s_t^\theta(\xx;x_0)
:=
\nabla\log q_t^{x_0}(\xx) - \nabla\log p_t^\theta(\xx)
\]
we also define the following terms:
\begin{equation*}
    \ell_\varepsilon(\theta;x_0)
    :=
    \E_{\xx_\varepsilon\sim q_\varepsilon^{x_0}}
    \bigl[-\log p_\varepsilon^\theta(\xx_\varepsilon)\bigr],
    \quad
    \mathcal J_w^{[\varepsilon,1]}(\theta;x_0)
    :=
    \int_\varepsilon^1
    w(t)\,
    \E_{\xx_t\sim q_t^{x_0}}
    \bigl[
    \|\Delta v_t^\theta(\xx_t;x_0)\|^2
    \bigr]\,dt
\end{equation*}

\begin{assumption}[Pathwise regularity]
    \label{ass:pw-pathwise-regularity}
    For the fixed $x_0$, the densities $q_t^{x_0}$ and $p_t^\theta$ are strictly positive and $C^1$ in $x$ for $t\in(0,1)$.  Their continuity equations hold in strong form, $t\mapsto\mathrm{KL}(q_t^{x_0}\|p_t^\theta)$ is differentiable, and the integrations by parts used below have no boundary terms.
\end{assumption}
    
\begin{assumption}[Endpoint regularity]
    \label{ass:pw-endpoint-regularity}
    The map $f_\theta(x,t):=-\log p_t^\theta(x)$ is jointly continuous at
    $(x_0,0)$.  Moreover, for some $\varepsilon_0>0$, $C>0$, and $m\ge1$,
    \begin{equation*}
    |f_\theta(x,t)|\le C(1+\|x\|^m),
    \qquad
    x\in\R^d,\quad t\in[0,\varepsilon_0].
    \end{equation*}
\end{assumption}

\begin{theorem}[Pointwise NLL is a CFM term plus residuals~\citep{han2026conditional}]
    \label{thm:pw-practical-decomposition}
    Suppose \assumpref{ass:pw-pathwise-regularity} and \assumpref{ass:pw-endpoint-regularity} hold, and suppose that $p_1^\theta=q_1^{x_0}=\cN(0,I_d)$.
    Then, for every positive weight $w$ and every $\varepsilon\in(0,1)$,
    \begin{align}
    \ell_\varepsilon(\theta;x_0)
    &=
    H(q_\varepsilon^{x_0})
    +
    \int_\varepsilon^1
    \E_{q_t^{x_0}}
    \bigl[
    \langle \Delta v_t^\theta(\xx_t;x_0), \Delta s_t^\theta(\xx_t;x_0)\rangle
    \bigr]\,dt,
    \label{eq:pw-exact-identity}\\
    -\log p_0^\theta(x_0)
    &=
    H(q_\varepsilon^{x_0})
    +
    \mathcal J_w^{[\varepsilon,1]}(\theta;x_0)
    +
    \mathcal G_{\varepsilon,w}(\theta;x_0)
    +
    \mathcal B_\varepsilon(\theta;x_0),
    \label{eq:pw-practical-decomposition}
    \end{align}
    where
    \begin{align*}
    &\mathcal G_{\varepsilon,w}(\theta;x_0) :=\int_\varepsilon^1\E_{q_t^{x_0}}\bigl[\langle \Delta v_t^\theta(\xx_t;x_0), \Delta s_t^\theta(\xx_t;x_0)-w(t)\Delta v_t^\theta(\xx_t;x_0)\rangle\bigr]\,dt,\\
    H(q_\varepsilon^{x_0})&=\frac d2\log(2\pi e\,\varepsilon^2), \quad 
    \mathcal B_\varepsilon(\theta;x_0) :=-\log p_0^\theta(x_0)-\ell_\varepsilon(\theta;x_0), \quad
    \lim_{\varepsilon\downarrow0}\mathcal B_\varepsilon(\theta;x_0)=0.
    \end{align*}
\end{theorem}

\subsection{Expectation-level Relation between Log-Likelihood and Flow Matching Loss}

Throughout this subsection, let
\[
t\sim\operatorname{Unif}(0,1),\qquad
\xx_0\sim p_{\mathrm{data}},\qquad
\xx_1\sim\cN(0,I_d),
\]
where $\xx_0$ and $\xx_1$ are independent, and let
\[
\xx_t=(1-t)\xx_0+t\xx_1\sim q_t.
\]
The marginal velocity field of $(q_t)_{t\in[0,1]}$ is
\[
u_t(x):=\E[\xx_1-\xx_0\mid \xx_t=x].
\]

\begin{lemma}[Optimal solution of Flow Matching Loss]
    \label{lem:optimal_solution_fm}
    Let us consider the following optimization problem:
    \begin{equation}
        \min_{\theta} \E_{t, \xx_0, \xx_1} \left[ \| v_\theta((1-t)\xx_0+t\xx_1, t) - (\xx_1 - \xx_0) \|_2^2 \right]
    \end{equation}
    Then, for every $t\in(0,1)$, the optimal solution is given by
    \begin{equation}
        v^\star(x, t) = \E_{\xx_0, \xx_1} \left[  \xx_1 - \xx_0 \mid \xx_t = x \right]
    \end{equation}
    where $\xx_t = (1-t)\xx_0+t\xx_1$.
\end{lemma}
\begin{proof}
    \begin{align*}
        &\E_{t, \xx_0, \xx_1} \left[ \| v_\theta((1-t)\xx_0+t\xx_1, t) - (\xx_1 - \xx_0) \|_2^2 \right] \\
        &= \E_{t, \xx_0, \xx_1} \left[ \| v_\theta(\xx_t, t) - v^\star(\xx_t, t) + v^\star(\xx_t, t) - (\xx_1 - \xx_0) \|_2^2 \right] \\
        &= \E_{t, \xx_t} [\E_{\xx_0, \xx_1} \left[ \| v_\theta(\xx_t, t) - v^\star(\xx_t, t) + v^\star(\xx_t, t) - (\xx_1 - \xx_0) \|_2^2 \mid \xx_t \right] ] \\
        &\qquad \text{(by tower property of expectation)} \\
        &= \E_{t, \xx_t} [\E_{\xx_0, \xx_1} [ \| v_\theta(\xx_t, t) - v^\star(\xx_t, t)\|^2_2 + \|v^\star(\xx_t, t) - (\xx_1 - \xx_0) \|_2^2 \\
        &\qquad + 2\langle v_\theta(\xx_t, t) - v^\star(\xx_t, t), v^\star(\xx_t, t) - (\xx_1 - \xx_0) \rangle \mid \xx_t ] ]
    \end{align*}
    Now, we prove that $\E_{t, \xx_t} [\E_{\xx_0, \xx_1} \left[\langle v_\theta(\xx_t, t) - v^\star(\xx_t, t), v^\star(\xx_t, t) - (\xx_1 - \xx_0) \rangle \mid \xx_t \right] ] = 0$.
    \begin{align*}
        &\E_{t, \xx_t} [\E_{\xx_0, \xx_1} \left[\langle v_\theta(\xx_t, t) - v^\star(\xx_t, t), v^\star(\xx_t, t) - (\xx_1 - \xx_0) \rangle \mid \xx_t \right] ]\\
        &= \E_{t, \xx_t} [\langle v_\theta(\xx_t, t) - v^\star(\xx_t, t), v^\star(\xx_t, t) - \E_{\xx_0, \xx_1} \left[(\xx_1 - \xx_0)\mid \xx_t \right]  \rangle ] \\
        &\qquad \text{(by the linearity of expectation)} \\
        &= \E_{t, \xx_t} [\langle v_\theta(\xx_t, t) - v^\star(\xx_t, t), v^\star(\xx_t, t) - v^\star(\xx_t, t)  \rangle ] \quad \text{(by $v^\star(\xx_t, t) = \E_{\xx_0, \xx_1} \left[(\xx_1 - \xx_0)\mid \xx_t \right]$)} \\
        &= 0
    \end{align*}
    Therefore,
    \begin{align*}
        &\E_{t, \xx_0, \xx_1} \left[ \| v_\theta((1-t)\xx_0+t\xx_1, t) - (\xx_1 - \xx_0) \|_2^2 \right] \\
        &= \E_{t, \xx_0, \xx_1} [ \| v_\theta(\xx_t, t) - v^\star(\xx_t, t)\|^2_2 + \|v^\star(\xx_t, t) - (\xx_1 - \xx_0) \|_2^2 ]
    \end{align*}
    Clearly, the optimal solution is attained when $v_\theta(x, t) = v^\star(x, t)$.
\end{proof}

\begin{lemma}
    \label{lem:optimal_solution_fm_score}
    Under the optimal solution $v^\star(x, t)$, we have:
    \begin{equation}
        v^\star(x, t) = -\frac{1}{1-t} x - \frac{t}{1-t} \nabla \log q_t(x)
    \end{equation}
    where $\xx_t = (1-t)\xx_0 + t\xx_1 \sim q_t(x)$.
\end{lemma}

\begin{proof}
Let $q^{x_0}_t(x) = q_t(x \mid x_0)$. Because $\xx_t = (1-t) \xx_0 + t \xx_1$ and $\xx_1 \sim \cN(0, I)$, then
\begin{equation}
    q_{t}^{x_0}(\xx_t) = \mathcal{N}((1 - t)x_0, t^2 I)
\end{equation}
and we have:
\begin{align}
    \nabla \log q_{t}^{x_0}(\xx_t) 
    &= -\frac{1}{t} \xx_1
\end{align}
Therefore,
\begin{align}
    \nabla \log q_t(\xx_t) 
    &= \E[\nabla \log q_{t}^{x_0}(\xx_t) \mid \xx_t] = -\frac{1}{t} \E[\xx_1 \mid \xx_t]
\end{align}
Then, for the optimal solution $v^\star(x, t)$, we have:
\begin{align}
    v^\star(x, t) 
    &= \E[\xx_1 - \xx_0 \mid \xx_t = x]\\
    &= \E[\xx_1 - \frac{\xx_t - t\xx_1}{1-t} \mid \xx_t = x] \\
    &= -\frac{1}{1-t} x + \frac{1}{1-t} \E[\xx_1 \mid \xx_t = x] \\
    &= -\frac{1}{1-t} x - \frac{t}{1-t} \nabla \log q_{t}(x)
\end{align}
\end{proof}

\begin{theorem}
    \label{thm:expectation-level-relation1}
    Suppose \assumpref{ass:pw-pathwise-regularity} and \assumpref{ass:pw-endpoint-regularity} hold and let $p_{\mathrm{data}}(\xx_0)$ be the data distribution. Then,
    \begin{align*}
    &\E_{\xx_0 \sim p_{\mathrm{data}}(\xx_0)} [-\log p_0^\theta(\xx_0)] \\
    &\le 
    H(p_{\mathrm{data}})
    +
    \sqrt{\E\bigl[\| v_\theta(\xx_t, t) - (\xx_1 - \xx_0)\|^2_2\bigr] \cdot
    \E\bigl[\| \nabla \log q_t(\xx_t) - \nabla \log p^\theta_t(\xx_t)\|^2_2\bigr] }
    \end{align*}
In the unrestricted, well-specified population setting, the marginal target field $v^\star(\xx_t,t)=\E[\xx_1-\xx_0\mid\xx_t]$ minimizes the flow matching loss and induces the data endpoint distribution, which also minimizes cross-entropy.
\end{theorem}

\begin{proof}
By Cauchy-Schwarz inequality, we have:
\begin{align}
    &(\E_{t, \xx_0, \xx_t}[\langle v_\theta(\xx_t, t) - (\xx_1 - \xx_0), \nabla \log q_t(\xx_t) - \nabla \log p^\theta_t(\xx_t) \rangle])^2 \\
    &\le \E_{t, \xx_0, \xx_t}[\| v_\theta(\xx_t, t) - (\xx_1 - \xx_0)\|^2_2] \cdot \E_{t, \xx_0, \xx_t}[\|\nabla \log q_t(\xx_t) - \nabla \log p^\theta_t(\xx_t)\|^2_2] 
\end{align}

Then, by \thmref{thm:pw-practical-decomposition}, we have:
\begin{align}
    &\E_{\xx_0 \sim p_{\mathrm{data}}(\xx_0)} [-\log p_0^\theta(\xx_0)] \\
    &= 
    H(p_{\mathrm{data}})
    +
    \E
    \bigl[
    \langle v_\theta(\xx_t, t) - (\xx_1 - \xx_0), \nabla \log q_t(\xx_t) - \nabla \log p^\theta_t(\xx_t)\rangle
    \bigr] \\
    &\le 
    H(p_{\mathrm{data}})
    +
    (\E
    \bigl[
    \| v_\theta(\xx_t, t) - (\xx_1 - \xx_0)\|^2_2
    \bigr] \E
    \bigl[
    \| \nabla \log q_t(\xx_t) - \nabla \log p^\theta_t(\xx_t)\|^2_2
    \bigr] )^{1/2}
\end{align}

By \lemref{lem:optimal_solution_fm}, the unrestricted population minimizer is
$v^\star(\xx_t,t)=\E[\xx_1-\xx_0\mid\xx_t]$.
This field generates the prescribed marginal path $q_t$, so $p_t^{\theta^\star}=q_t$ in the well-specified case and the score-gap term vanishes.
Consequently the induced endpoint distribution equals $p_{\mathrm{data}}$, which minimizes cross-entropy.
This population statement does not imply equality of the two objectives, nor identical minimizers in a restricted parameter class.

\end{proof}

\begin{theorem}[Flow matching upper bound under uniform regularity]
\label{thm:flow-matching-upper-bound-uniform-regularity}
Suppose \assumpref{ass:pw-pathwise-regularity} and
\assumpref{ass:pw-endpoint-regularity} hold, and let
$p_1^\theta=\cN(0,I_d)$. Assume
$M_2:=\E_{p_{\mathrm{data}}}[\|\xx_0\|^2]<\infty$ and that there
exist constants $C_0,C_1,C_2<\infty$, independent of $\theta$, such
that, for every $x\in\mathbb{R}^d$ and $t\in[0,1]$,
\begin{equation}
    \|v_\theta(0,t)\|\le C_0,\qquad
    \|\nabla v_\theta(x,t)\|\le C_1,\qquad
    \|\nabla^\top\nabla\cdot v_\theta(x,t)\|\le C_2.
\end{equation}
Assume additionally that the data-path score is uniformly square-integrable,
\begin{equation}
    C_q:=\sup_{t\in(0,1]}\E_{q_t}[\|\nabla\log q_t(\xx_t)\|^2]<\infty.
\end{equation}
Then, there exists a constant $C_s<\infty$, independent of $\theta$, such that:
\begin{equation}
    \label{eq:expect_relation_nll_fm}
    \E_{\xx_0\sim p_{\mathrm{data}}}
    [-\log p_0^\theta(\xx_0)]
    \le
    \E\bigl[\|v_\theta(\xx_t,t)-(\xx_1-\xx_0)\|^2\bigr]
    +H(p_{\mathrm{data}})+\frac{C_s}{4}.
\end{equation}
In the unrestricted, well-specified population setting, the marginal target field minimizes both sides.
\end{theorem}

\begin{proof}
Let $s^\theta_t(x) = \nabla \log p^\theta_t(x)$, then we consider the continuity equation:
\begin{align}
    \partial_t p^\theta_t(x) + \nabla \cdot (p^\theta_t(x) v_\theta(x, t)) &= 0 \\
    \partial_t p^\theta_t(x) + \nabla p^\theta_t(x) \cdot v_\theta(x, t) + p^\theta_t(x) \nabla \cdot v_\theta(x, t) &= 0 \\
    \partial_t \log p^\theta_t(x) + v_\theta(x, t) \cdot \nabla \log p^\theta_t(x)  + \nabla \cdot v_\theta(x, t) &= 0
\end{align}

Taking $\nabla$ on both sides, we get:
\begin{align}
    \partial_t \nabla\log p^\theta_t(x) + \nabla v_\theta(x, t) \cdot \nabla \log p^\theta_t(x) + v_\theta(x, t) \nabla^\top \nabla \log p^\theta_t(x)  + \nabla^\top \nabla \cdot v_\theta(x, t) &= 0 \\
    \partial_t s^\theta_t(x) + \nabla v_\theta(x, t) \cdot s^\theta_t(x) + v_\theta(x, t) \nabla^\top s^\theta_t(x)  + \nabla^\top \nabla \cdot v_\theta(x, t) &= 0
\end{align}
Let $X_t$ denote the trajectory, then we have:
\begin{align}
    \frac{d}{dt} s^\theta_t(X_t) + \nabla v_\theta(X_t, t) \cdot s^\theta_t(X_t)  + \nabla^\top \nabla \cdot v_\theta(X_t, t) = 0 \\
    \frac{d}{dt} s^\theta_t(X_t) = - \nabla v_\theta(X_t, t) \cdot s^\theta_t(X_t)  - \nabla^\top \nabla \cdot v_\theta(X_t, t) \label{eq:score_equation_linear}
\end{align}

Therefore, we need to control $\|\nabla v_\theta\| $ and $\| \nabla^\top \nabla \cdot v_\theta \| $. We assume $v_\theta$ is smooth and constrained in a compact set, then we can have:
\begin{align}
    \|\nabla v_\theta\| \le C_1 \\
    \| \nabla^\top \nabla \cdot v_\theta \| \le C_2
\end{align}

We further assume that $\|v_\theta(0,t)\|\le C_0$ uniformly over the parameter set. Then, we define $g(\lambda) = v_\theta(\lambda x, t)$, $\lambda \in [0,1]$, and we have:
\begin{align}
    g(1) - g(0) &= \int_0^1 \frac{d}{d\lambda} g(\lambda) d\lambda
\end{align}
This is actually:
\begin{align}
    v_\theta(x, t) - v_\theta(0, t) = \int_0^1 \nabla v_\theta(\lambda x, t) x d\lambda
\end{align}
Therefore, we have:
\begin{align}
    \|v_\theta(x, t)\| &= \| v_\theta(0, t) + \int_0^1 \nabla v_\theta(\lambda x, t) x d\lambda \| \\
    &\le \| v_\theta(0, t)\| + \| \int_0^1 \nabla v_\theta(\lambda x, t) x d\lambda \| \\
    &\le C_0 + \int_0^1 \|\nabla v_\theta(\lambda x, t)\| \|x\| d\lambda \\
    &\le C_0 + C_1 \|x\|
\end{align}

Let $X_r$ be the characteristic starting from $X_t=x$, namely,
\begin{equation}
    X_r=x+\int_t^r v_\theta(X_s,s)\,ds,
    \qquad r\in[t,1].
\end{equation}
Using the linear-growth bound above, we obtain
\begin{align}
    \|X_r\|
    &\le \|x\| + \int_t^r \|v_\theta(X_s, s)\|\,ds \\
    &\le \|x\| + \int_t^r (C_0 + C_1 \|X_s\|)\,ds.
\end{align}
Therefore, Gr\"onwall's inequality gives
\begin{equation}
    \|X_r\|
    \le
    e^{C_1(r-t)}\|x\|
    +\frac{C_0}{C_1}\bigl(e^{C_1(r-t)}-1\bigr),
    \qquad r\in[t,1],
    \label{eq:linear_growth_bound}
\end{equation}
Since $p_1^\theta=\cN(0,I)$, we have $s_1^\theta(x)=-x$. Then, by \eqref{eq:score_equation_linear}, we have:
\begin{align}
    \frac{d}{dt} s^\theta_t(X_t) &= - \nabla v_\theta(X_t, t) \cdot s^\theta_t(X_t)  - \nabla^\top \nabla \cdot v_\theta(X_t, t) \\
    \int^1_t d s^\theta_t(X_t) &= \int^1_t (- \nabla v_\theta(X_t, t) \cdot s^\theta_t(X_t)  - \nabla^\top \nabla \cdot v_\theta(X_t, t)) dt \\
    s^\theta_1(X_1) - s^\theta_t(X_t) &= \int^1_t (- \nabla v_\theta(X_t, t) \cdot s^\theta_t(X_t)  - \nabla^\top \nabla \cdot v_\theta(X_t, t)) dt \\
    \|s_t^\theta(X_t)\|
    &\le \|s_1^\theta(X_1)\|
    +\int_t^1
    \left(C_1\|s_r^\theta(X_r)\|+C_2\right)\,dr \\
    &=\|X_1\|
    +\int_t^1
    \left(C_1\|s_r^\theta(X_r)\|+C_2\right)\,dr.
\end{align}
Applying Gr\"onwall's inequality once more and using \eqref{eq:linear_growth_bound} gives
\begin{align}
    \|s_t^\theta(x)\|
    &\le
    e^{C_1(1-t)}\|X_1\|
    +\frac{C_2}{C_1}\bigl(e^{C_1(1-t)}-1\bigr) \\
    &\le K_0+K_1\|x\|,
\end{align}
where $K_0,K_1<\infty$ depend only on $C_0,C_1,C_2$ and are
independent of $\theta$.  If
$M_2:=\E_{p_{\mathrm{data}}}[\|\xx_0\|^2]<\infty$, then
\begin{align}
    \E[\|\xx_t\|^2] &= \E[\|(1-t)\xx_0 + t\xx_1\|^2] \\
    &= \E[(1-t)^2 \|\xx_0\|^2 + t^2 \|\xx_1\|^2 + 2t(1-t) \langle \xx_0, \xx_1 \rangle] \\
    &= (1-t)^2 \E[\|\xx_0\|^2] + t^2 \E[\|\xx_1\|^2] + 2t(1-t) \E[\langle \xx_0, \xx_1 \rangle] \\
    &= (1-t)^2 \E[\|\xx_0\|^2] + t^2 \E[\|\xx_1\|^2] + 2t(1-t) \langle \E[\xx_0], \E[\xx_1] \rangle] \quad \text{(by $\xx_0 \perp \xx_1$)} \\
    &= (1-t)^2 \E \|\xx_0\|^2 + t^2 \E \|\xx_1\|^2 \quad \text{(by $\E[\xx_1] = 0$)} \\
    &= (1-t)^2 M_2 + t^2 d \quad \text{(by $\xx_1 \sim \cN(0,I_d)$)} \\
    &\le M_2 + d \quad \text{(by $t \in [0,1]$)}
\end{align}
and consequently
\begin{equation}
    \sup_{\theta,t\in[0,1]}
    \E_{q_t}[\|\nabla\log p_t^\theta(\xx_t)\|^2]
    \le
    2K_0^2+2K_1^2(M_2+d)
    =:C_p<\infty.
\end{equation}

It remains to control the score of the data path.  By the identity in \lemref{lem:optimal_solution_fm_score},
\begin{equation}
    \nabla\log q_t(\xx_t)
    =\E\left[-\frac{\xx_1}{t}\,\middle|\,\xx_t\right].
\end{equation}
Hence, on every truncated interval $t\in[\varepsilon,1]$,
\begin{align}
    \E_{q_t}[\|\nabla\log q_t(\xx_t)\|^2]
    &= \E_{q_t}[\| -\frac{1}{t}\E[\xx_1 \mid \xx_t] \|^2] \\
    &= \frac{1}{t^2} \E_{q_t}[\| \E[\xx_1 \mid \xx_t] \|^2] \\
    &\le \frac{1}{t^2} \E_{q_t}[\E[\| \xx_1 \|^2\mid \xx_t]] \\
    &= \frac{1}{t^2} \E[\| \xx_1 \|^2] \quad \text{(by tower property)} \\
    &\le \frac{d}{\varepsilon^2}.
\end{align}
For the full interval, we use the endpoint integrability condition in the theorem,
\begin{equation}
    C_q:=\sup_{t\in(0,1]}
    \E_{q_t}[\|\nabla\log q_t(\xx_t)\|^2]<\infty.
\end{equation}
Combining the last two bounds yields
\begin{equation}
    \E\bigl[
    \|\nabla\log q_t(\xx_t)-\nabla\log p_t^\theta(\xx_t)\|^2
    \bigr]
    \le 2(C_q+C_p)=:C_s<\infty,
\end{equation}
uniformly over $\theta$.  Substituting this estimate into
\thmref{thm:expectation-level-relation1}, and then using Young's
inequality, gives, for every $\kappa>0$,
\begin{align}
    \E_{\xx_0\sim p_{\mathrm{data}}}
    [-\log p_0^\theta(\xx_0)]
    &\le
    H(p_{\mathrm{data}})
    +\sqrt{C_s\,
    \E[\|v_\theta(\xx_t,t)-(\xx_1-\xx_0)\|^2]} \\
    &\le
    \kappa\,
    \E[\|v_\theta(\xx_t,t)-(\xx_1-\xx_0)\|^2]
    +H(p_{\mathrm{data}})+\frac{C_s}{4\kappa}.
\end{align}
Taking $\kappa=1$ proves the bound. The population-minimizer statement follows from \thmref{thm:expectation-level-relation1} under the stated well-specified interpretation.
\end{proof}

\section{Theoretical Proofs and Results}

\subsection{Idealized Posteriors and Empirical Source Distributions}
\label{app:source_distributions}

The main text defines $\pi^+$ and $\pi^-$ as posterior target distributions induced by the binary optimality variable. In offline preference optimization, however, the data available for training are usually fixed empirical source distributions $\hat{\pi}^+$ and $\hat{\pi}^-$ constructed by a separate pipeline, such as best-of-$N$ winner/loser selection from a frozen generator. The formal identities in the main text only require that the source distributions be fixed during optimization. The practical modeling question is therefore whether the offline construction is a faithful enough surrogate for the idealized targets. In our experiments, the data construction pipeline is frozen, so the resulting preferred and dispreferred distributions remain stationary throughout training.

\subsection{Proof of Theorem~\ref{thm:rlhf_kl}}
\label{app:rlhf_kl}

\begin{proof}
For each context $c$, \eqref{eq:reward} gives
\[
r(\xx_0, c)
=
\frac{1}{2\omega}
\left[
\log \pi^+(\xx_0 \mid c)
- \log \pi^-(\xx_0 \mid c)
+ \log \frac{p_{\pi_{\text{ref}}}(o=1\mid c)}{p_{\pi_{\text{ref}}}(o=0\mid c)}
\right].
\]
Taking expectation under $\pi_\theta(\cdot \mid c)$ yields
\begin{align*}
\E_{\xx_0 \sim \pi_\theta(\cdot \mid c)} [r(\xx_0, c)]
=&
\frac{1}{2\omega}
\E_{\xx_0 \sim \pi_\theta(\cdot \mid c)}
\big[
\log \pi^+(\xx_0 \mid c) - \log \pi^-(\xx_0 \mid c)
\big] \\
&\quad
+ \frac{1}{2\omega}
\log \frac{p_{\pi_{\text{ref}}}(o=1\mid c)}{p_{\pi_{\text{ref}}}(o=0\mid c)}.
\end{align*}
The second term depends only on $c$, hence is irrelevant for optimization over $\theta$. For the first term,
\begin{align*}
\E_{\pi_\theta}[\log \pi^+]
&= -\mathcal{D}_{KL}(\pi_\theta \| \pi^+) - \mathcal{H}(\pi_\theta), \\
\E_{\pi_\theta}[\log \pi^-]
&= -\mathcal{D}_{KL}(\pi_\theta \| \pi^-) - \mathcal{H}(\pi_\theta),
\end{align*}
where $\mathcal{H}(\pi_\theta)$ is the conditional differential entropy of $\pi_\theta(\cdot \mid c)$. Subtracting the two identities cancels the entropy term:
\[
\E_{\pi_\theta}[\log \pi^+ - \log \pi^-]
=
- \mathcal{D}_{KL}(\pi_\theta \| \pi^+) + \mathcal{D}_{KL}(\pi_\theta \| \pi^-).
\]
Taking expectation over $c$ and dropping the positive constant $\frac{1}{2\omega}$ proves the claim.
\end{proof}

\subsection{An Idealized Interpretation of Interpolated Vector Fields}
\label{app:interpolation_vector_field}

Proposition 1 of \cite{holderrieth2025introduction} gives the following relation when a velocity field and a score describe the same Gaussian marginal path (see also \lemref{lem:optimal_solution_fm_score}):
\begin{equation}
    u_t(\xx) = -\frac{t}{1-t} \nabla \log p_t(\xx) - \frac{1}{1-t} \xx.
\end{equation}
For an idealized interpretation, suppose that both $v_\theta$ and $v_{\mathrm{old}}$ satisfy this compatibility condition:
\begin{align}
    v_\theta(\xx_t, t) &= -\frac{t}{1-t} \nabla \log p^\theta_t(\xx_t) - \frac{1}{1-t} \xx_t, \\
    v_{\mathrm{old}}(\xx_t, t) &= -\frac{t}{1-t} \nabla \log p^{\mathrm{old}}_t(\xx_t) - \frac{1}{1-t} \xx_t.
\end{align}
Then the positive interpolation satisfies
\begin{align}
    v^+(\xx_t, t) &= v_{\mathrm{old}}(\xx_t, t) + \beta (v_\theta(\xx_t, t) - v_{\mathrm{old}}(\xx_t, t)) \\
    &= -\frac{t}{1-t} \nabla \log\left[ p^{\mathrm{old}}_t(\xx_t) \left(\frac{p^{\theta}_t(\xx_t)}{p^{\mathrm{old}}_t(\xx_t)}\right)^\beta \right] - \frac{1}{1-t} \xx_t.
\end{align}
Similarly,
\begin{align}
    v^-(\xx_t, t) &= v_{\mathrm{old}}(\xx_t, t) - \beta (v_\theta(\xx_t, t) - v_{\mathrm{old}}(\xx_t, t)) \\
    &= -\frac{t}{1-t} \nabla \log\left[ p^{\mathrm{old}}_t(\xx_t) \left(\frac{p^{\mathrm{old}}_t(\xx_t)}{p^{\theta}_t(\xx_t)}\right)^{\beta} \right] - \frac{1}{1-t} \xx_t.
\end{align}
Thus, at each fixed $t$, $v^+$ and $v^-$ admit a score-field interpretation analogous to the CFG~\citep{luo2022understanding,ho2022classifier} interpolation $p_t(\xx_t, \emptyset)\left( \frac{p_t(\xx_t, c)}{p_t(\xx_t, \emptyset)} \right)^\beta$ over condition space.

\subsection{Proof of Theorem~\ref{thm:flowcpo_bound}}
\label{app:fm_bound}

\begin{proof}[Proof of \thmref{thm:flowcpo_bound}]
Expanding each KL term in \eqref{eq:op_kl} into cross-entropy plus entropy gives
\begin{align*}
\mathcal{L}_{\text{Offline}}(\theta)
=&
\E_c \Big[
\E_{\xx_0^+ \sim \pi^+(\cdot \mid c)}[-\log \pi_\theta^+(\xx_0^+ \mid c)]
+ \lambda \E_{\xx_0^- \sim \pi^-(\cdot \mid c)}[-\log \pi_\theta^-(\xx_0^- \mid c)]
\Big]
+ C_{\mathrm{ent}},
\end{align*}
where $C_{\mathrm{ent}} = -\E_c[\mathcal{H}(\pi^+(\cdot \mid c)) + \lambda \mathcal{H}(\pi^-(\cdot \mid c))]$ is independent of $\theta$.

From \thmref{thm:flow-matching-upper-bound-uniform-regularity}, we have:
\begin{equation}
    \E_{\xx_0\sim p_{\mathrm{data}}}[-\log p_0^\theta(\xx_0)]
    \le
    \E\bigl[\|v_\theta(\xx_t,t)-(\xx_1-\xx_0)\|^2\bigr] + C.
\end{equation}
Therefore, we have:
\begin{align*}
    \E_{\xx_0^+ \sim \pi^+(\cdot \mid c)}[-\log \pi_\theta^+(\xx_0^+ \mid c)]
    &\le
    C_+(c) + \E_{\substack{\xx_0^+ \sim \pi^+(\cdot \mid c) \\ t,\, \epsilon}}[\|\mu_\theta(\xx_t^+, t, c) - u_t(\xx_t^+ \mid \xx_0^+)\|_2^2], \\
    \E_{\xx_0^- \sim \pi^-(\cdot \mid c)}[-\log \pi_\theta^-(\xx_0^- \mid c)]
    &\le
    C_-(c) + \E_{\substack{\xx_0^- \sim \pi^-(\cdot \mid c) \\ t,\, \epsilon}}[\|\nu_\theta(\xx_t^-, t, c) - u_t(\xx_t^- \mid \xx_0^-)\|_2^2],
    \end{align*}
Then, 
\begin{align*}
    \mathcal{L}_{\text{Offline}}(\theta)
    &\le 
    \E_{\substack{c, \xx_0^+ \sim \pi^+(\cdot \mid c) \\ t,\, \epsilon}}[\|\mu_\theta(\xx_t^+, t, c) - u_t(\xx_t^+ \mid \xx_0^+)\|_2^2] \\
    &\quad+
    \lambda \cdot \E_{\substack{c, \xx_0^- \sim \pi^-(\cdot \mid c) \\ t,\, \epsilon}}[\|\nu_\theta(\xx_t^-, t, c) - u_t(\xx_t^- \mid \xx_0^-)\|_2^2] + Constant.
\end{align*}
This establishes an upper-bound surrogate for \eqref{eq:op_kl}. If the two branch fields can simultaneously realize their population targets, then both branch endpoint distributions match their corresponding sources.
\end{proof}

\subsection{Relation between FlowCPO and Simplified FlowDPO}
\label{app:flowdpo_ema_connection}

We isolate the contrastive core of FlowDPO by omitting its sigmoid weighting and frozen-reference terms. Using the preferred/dispreferred notation from \eqref{eq:final_loss}, the simplified objective is
\begin{equation}
    \min_\theta \mathcal L_{\mathrm{FlowDPO}}^{\mathrm{simple}}(\theta) = \E\Big[
        \|v_\theta(\xx_t^w,t,c)-(\epsilon-\xx_0^w)\|_2^2
        -\|v_\theta(\xx_t^l,t,c)-(\epsilon-\xx_0^l)\|_2^2
    \Big].
    \label{eq:simplified_flowdpo_margin}
\end{equation}
It captures the per-pair contrastive gradient direction, without retaining the full FlowDPO objective. Below, $\E$ averages over the same fixed offline pairs and sampled $t,\epsilon$ as in \eqref{eq:final_loss}. We set $\lambda=1$ to match the unweighted preference comparison above.

\begin{theorem}[Relation between FlowCPO and Simplified FlowDPO]
\label{thm:flowcpo_simple_flowdpo}
Let $\beta>0$, $\lambda=1$, and hold $v_{\mathrm{old}}$ fixed during each gradient step. Assuming the displayed expectations are finite, the FlowCPO loss satisfies
\begin{equation}
    \begin{aligned}
        \mathcal L_{\text{\normalfont\method}}(\theta)
        ={}&\beta\mathcal L_{\mathrm{FlowDPO}}^{\mathrm{simple}}(\theta)\\
        &+\beta(\beta-1)\E\big[
            \|v_\theta(\xx_t^w,t,c)-v_{\mathrm{old}}(\xx_t^w,t,c)\|_2^2
          \big]\\
        &+\beta(\beta+1)\E\big[
            \|v_\theta(\xx_t^l,t,c)-v_{\mathrm{old}}(\xx_t^l,t,c)\|_2^2
          \big]+C,
    \end{aligned}
    \label{eq:flowcpo_simple_flowdpo}
\end{equation}
where $C$ is independent of $\theta$ for the current detached $v_{\mathrm{old}}$.
\end{theorem}

\begin{proof}
The preferred branch uses $\xx_t^w$ and target $u_t(\xx_t^w\mid\xx_0^w)=\epsilon-\xx_0^w$, while the dispreferred branch uses $\xx_t^l$ and target $u_t(\xx_t^l\mid\xx_0^l)=\epsilon-\xx_0^l$. We expand each branch separately.

\textbf{Preferred branch.}
Substituting $\mu_\theta=(1-\beta)v_{\mathrm{old}}+\beta v_\theta$ and collecting terms gives
\begin{align*}
    &\mu_\theta(\xx_t^w,t,c)-u_t(\xx_t^w\mid\xx_0^w)\\
    &\quad=v_{\mathrm{old}}(\xx_t^w,t,c)-u_t(\xx_t^w\mid\xx_0^w)
        +\beta\big[v_\theta(\xx_t^w,t,c)-v_{\mathrm{old}}(\xx_t^w,t,c)\big].
\end{align*}
Expanding the squared norm yields
\begin{align*}
    &\|\mu_\theta(\xx_t^w,t,c)-u_t(\xx_t^w\mid\xx_0^w)\|_2^2\\
    &\quad=\|v_{\mathrm{old}}(\xx_t^w,t,c)-u_t(\xx_t^w\mid\xx_0^w)\|_2^2\\
    &\qquad+\beta^2\|v_\theta(\xx_t^w,t,c)-v_{\mathrm{old}}(\xx_t^w,t,c)\|_2^2\\
    &\qquad+2\beta\Big\langle v_{\mathrm{old}}(\xx_t^w,t,c)-u_t(\xx_t^w\mid\xx_0^w),
        v_\theta(\xx_t^w,t,c)-v_{\mathrm{old}}(\xx_t^w,t,c)\Big\rangle.
\end{align*}
The residual of the current model is the sum of the EMA residual and $v_\theta(\xx_t^w,t,c)-v_{\mathrm{old}}(\xx_t^w,t,c)$. Expanding the squared norm of this sum and solving for the cross term gives
\begin{align*}
    &2\Big\langle v_{\mathrm{old}}(\xx_t^w,t,c)-u_t(\xx_t^w\mid\xx_0^w),
        v_\theta(\xx_t^w,t,c)-v_{\mathrm{old}}(\xx_t^w,t,c)\Big\rangle\\
    &\quad=\|v_\theta(\xx_t^w,t,c)-u_t(\xx_t^w\mid\xx_0^w)\|_2^2\\
    &\qquad-\|v_{\mathrm{old}}(\xx_t^w,t,c)-u_t(\xx_t^w\mid\xx_0^w)\|_2^2
        -\|v_\theta(\xx_t^w,t,c)-v_{\mathrm{old}}(\xx_t^w,t,c)\|_2^2.
\end{align*}
Substituting this identity into the expansion gives coefficients $\beta$, $\beta^2-\beta=\beta(\beta-1)$, and $1-\beta$ for the three squared errors, respectively:
\begin{align*}
    \|\mu_\theta(\xx_t^w,t,c)-u_t(\xx_t^w\mid\xx_0^w)\|_2^2
    &=\beta\|v_\theta(\xx_t^w,t,c)-u_t(\xx_t^w\mid\xx_0^w)\|_2^2\\
    &\quad+\beta(\beta-1)\|v_\theta(\xx_t^w,t,c)-v_{\mathrm{old}}(\xx_t^w,t,c)\|_2^2\\
    &\quad+(1-\beta)\|v_{\mathrm{old}}(\xx_t^w,t,c)-u_t(\xx_t^w\mid\xx_0^w)\|_2^2.
\end{align*}

\textbf{Dispreferred branch.}
Substituting $\nu_\theta=(1+\beta)v_{\mathrm{old}}-\beta v_\theta$ now gives
\begin{align*}
    &\nu_\theta(\xx_t^l,t,c)-u_t(\xx_t^l\mid\xx_0^l)\\
    &\quad=v_{\mathrm{old}}(\xx_t^l,t,c)-u_t(\xx_t^l\mid\xx_0^l)
        -\beta\big[v_\theta(\xx_t^l,t,c)-v_{\mathrm{old}}(\xx_t^l,t,c)\big].
\end{align*}
The minus sign changes the sign of the cross term, while the coefficient of the squared distance remains $\beta^2$:
\begin{align*}
    &\|\nu_\theta(\xx_t^l,t,c)-u_t(\xx_t^l\mid\xx_0^l)\|_2^2\\
    &\quad=\|v_{\mathrm{old}}(\xx_t^l,t,c)-u_t(\xx_t^l\mid\xx_0^l)\|_2^2\\
    &\qquad+\beta^2\|v_\theta(\xx_t^l,t,c)-v_{\mathrm{old}}(\xx_t^l,t,c)\|_2^2\\
    &\qquad-2\beta\Big\langle v_{\mathrm{old}}(\xx_t^l,t,c)-u_t(\xx_t^l\mid\xx_0^l),
        v_\theta(\xx_t^l,t,c)-v_{\mathrm{old}}(\xx_t^l,t,c)\Big\rangle.
\end{align*}
Evaluating the same norm identity on the dispreferred input gives
\begin{align*}
    &2\Big\langle v_{\mathrm{old}}(\xx_t^l,t,c)-u_t(\xx_t^l\mid\xx_0^l),
        v_\theta(\xx_t^l,t,c)-v_{\mathrm{old}}(\xx_t^l,t,c)\Big\rangle\\
    &\quad=\|v_\theta(\xx_t^l,t,c)-u_t(\xx_t^l\mid\xx_0^l)\|_2^2\\
    &\qquad-\|v_{\mathrm{old}}(\xx_t^l,t,c)-u_t(\xx_t^l\mid\xx_0^l)\|_2^2
        -\|v_\theta(\xx_t^l,t,c)-v_{\mathrm{old}}(\xx_t^l,t,c)\|_2^2.
\end{align*}
Substituting this identity gives coefficients $-\beta$, $\beta^2+\beta=\beta(\beta+1)$, and $1+\beta$:
\begin{align*}
    \|\nu_\theta(\xx_t^l,t,c)-u_t(\xx_t^l\mid\xx_0^l)\|_2^2
    &=-\beta\|v_\theta(\xx_t^l,t,c)-u_t(\xx_t^l\mid\xx_0^l)\|_2^2\\
    &\quad+\beta(\beta+1)\|v_\theta(\xx_t^l,t,c)-v_{\mathrm{old}}(\xx_t^l,t,c)\|_2^2\\
    &\quad+(1+\beta)\|v_{\mathrm{old}}(\xx_t^l,t,c)-u_t(\xx_t^l\mid\xx_0^l)\|_2^2.
\end{align*}
Take expectations over the preferred and dispreferred samples, respectively, and add. The first terms give $\beta\mathcal L_{\mathrm{FlowDPO}}^{\mathrm{simple}}$, the second terms give the two corrections in \eqref{eq:flowcpo_simple_flowdpo}, and the last terms form $C$. The latter use their respective sample targets and are independent of $\theta$ because the data and $v_{\mathrm{old}}$ are fixed during differentiation.
\end{proof}

\paragraph{Interpretation.}
The theorem expresses FlowCPO as $\beta$ times simplified FlowDPO, plus two corrections determined by the distance between $v_\theta$ and $v_{\mathrm{old}}$ and a term $C$ that is constant during each gradient step. At $\beta=1$, only the quadratic constraint on dispreferred inputs remains, with coefficient $2$. For $0<\beta<1$, including our default $\beta=0.5$, the preferred correction has a negative coefficient, so the two corrections should not both be described as positive regularizers. The ablation analysis of $\beta$ is in \figref{fig:ablation_studies}.

\paragraph{Why use EMA for interpolation?}
We use EMA so that the interpolation reference follows the model being trained. \thmref{thm:flowcpo_simple_flowdpo} shows that the two correction terms depend on $v_\theta-v_{\mathrm{old}}$ at the corresponding preferred or dispreferred input. These corrections depend on the squared prediction differences, so a reference that falls far behind can substantially change the objective. If we use the frozen $v_{\mathrm{ref}}$ for interpolation, the reference stays at initialization while $v_\theta$ changes during fine-tuning. EMA updates the parameters of $v_{\mathrm{old}}$ from recent training iterates, with the aim of keeping this gap small. Thus, $v_{\mathrm{ref}}$ provides the fixed in-domain offline data, while $v_{\mathrm{old}}$ follows training and is used to form $\mu_\theta$ and $\nu_\theta$.

\paragraph{Advantages of FlowCPO over FlowDPO.}
For $\lambda=1$, the two losses can be compared directly:
\begin{equation}
    \begin{aligned}
        \mathcal L_{\mathrm{FlowDPO}}^{\mathrm{simple}}(\theta)
        &=\E\Big[
            \|v_\theta(\xx_t^w,t,c)-(\epsilon-\xx_0^w)\|_2^2
            -\|v_\theta(\xx_t^l,t,c)-(\epsilon-\xx_0^l)\|_2^2
          \Big],\\
        \mathcal L_{\text{\normalfont\method}}(\theta)
        &=\E\Big[
            \|\mu_\theta(\xx_t^w,t,c)-(\epsilon-\xx_0^w)\|_2^2
            +\|\nu_\theta(\xx_t^l,t,c)-(\epsilon-\xx_0^l)\|_2^2
          \Big].
    \end{aligned}
    \label{eq:flowcpo_simple_loss_comparison}
\end{equation}
The minus sign allows simplified FlowDPO to decrease without bound if the rejected error grows while the preferred error stays bounded. \method adds two nonnegative errors, so its loss is bounded below by zero. This removes one potential source of instability, although a lower bound alone does not guarantee stable training.

\newpage
\section{Connections to Existing Methods Under the Unified Divergence Framework}
\label{app:connections}

\Secref{subsec:connections} and \tabref{tab:comparison} summarize the connections under the unified divergence-based framework. 

\subsection{Connection to FlowGRPO}
\label{app:flowgrpo}

FlowGRPO~\citep{liu2025flow} is an online RL algorithm for flow matching models. Ignoring the KL term added in practice for stabilization, its core objective is
\[
\max_\theta \E_{c,\, \xx_0 \sim \pi_\theta(\cdot \mid c)} [r(\xx_0,c)]
\]
By \thmref{thm:rlhf_kl}, this is exactly equivalent to
\[
\min_\theta \E_c \left[
\mathcal{D}_{KL}(\pi_\theta \| \pi^+) - \mathcal{D}_{KL}(\pi_\theta \| \pi^-)
\right]
\]
Hence FlowGRPO belongs to the reverse-KL branch of \eqref{eq:generalized_obj} with
\[
q_\theta^+ = q_\theta^- = \pi_\theta,
\qquad
\alpha = 1,
\qquad
\gamma = -1,
\]
plus the usual stabilization term $\mathcal{D}_{KL}(\pi_\theta \| \pi_{\mathrm{ref}})$ used in practice.

\subsection{Connection to DiffusionNFT}
\label{app:connect_diffusionnft}

DiffusionNFT~\cite{zheng2025diffusionnft} optimizes a supervised loss of the form
\begin{align}
    \label{eq:diffusionnft_loss}
    \mathcal{L}_{\mathrm{NFT}}(\theta)
    &=
    \E_{\substack{c,\, \xx_0 \sim \pi_{\mathrm{old}}(\cdot \mid c) \\ t,\, \epsilon}}
    \Big[
    p(o=1 \mid \xx_0,c)\,
    \|\mu_\theta(\xx_t,t,c) - u_t(\xx_t \mid \xx_0)\|_2^2 \notag\\
    &\qquad\qquad\qquad\qquad\qquad
    +
    p(o=0 \mid \xx_0,c)\,
    \|\nu_\theta(\xx_t,t,c) - u_t(\xx_t \mid \xx_0)\|_2^2
    \Big].
\end{align}
Similar to \eqref{eq:posteriors}, we can have:
\begin{align*}
\pi_{\mathrm{old}}(\xx_0 \mid c)\, p(o=1 \mid \xx_0,c)
=
\pi^+(\xx_0 \mid c)\, p_{\pi_{\mathrm{old}}}(o=1 \mid c) \\
\pi_{\mathrm{old}}(\xx_0 \mid c)\, p(o=0 \mid \xx_0,c)
=
\pi^-(\xx_0 \mid c)\, p_{\pi_{\mathrm{old}}}(o=0 \mid c) 
\end{align*}
Therefore, with formula \eqref{eq:expect_relation_nll_fm} in \thmref{thm:flow-matching-upper-bound-uniform-regularity}, \eqref{eq:diffusionnft_loss} can be rewritten as
\begin{align*}
\mathcal{L}_{\mathrm{NFT}}(\theta)
=&
\E_c \Big[
p_{\pi_{\mathrm{old}}}(o=1 \mid c)\,
\E_{\substack{\xx_0 \sim \pi^+(\cdot \mid c) \\ t,\, \epsilon}}
\|\mu_\theta(\xx_t,t,c) - u_t\|_2^2
\Big] \\
&\quad
+
\E_c \Big[
p_{\pi_{\mathrm{old}}}(o=0 \mid c)\,
\E_{\substack{\xx_0 \sim \pi^-(\cdot \mid c) \\ t,\, \epsilon}}
\|\nu_\theta(\xx_t,t,c) - u_t\|_2^2
\Big] \\
&\ge
\E_c \Big[
    p_{\pi_{\mathrm{old}}}(o=1 \mid c)\,
    \E_{\pi^+}
    \left[ -\log \pi^+_\theta(\xx_0 \mid c) \right]\\
    &\qquad+
    p_{\pi_{\mathrm{old}}}(o=0 \mid c)\,
    \E_{\pi^-}
    \left[ -\log \pi^-_\theta(\xx_0 \mid c) \right]
    \Big] + C \\
    &= \E_c[p_{\pi_{\mathrm{old}}}(o=1 \mid c) D_{KL}(\pi^+ \| \pi^+_\theta) + p_{\pi_{\mathrm{old}}}(o=0 \mid c) D_{KL}(\pi^- \| \pi^-_\theta)] + C
\end{align*}
where $C$ is a constant independent of $\theta$. The both sides achieve the same optimal solution when $\mu_{\theta^\star}(\xx_t^+, t, c) = \E[\xx_1 - \xx_0^+\mid \xx_t^+, t, c]$ and $\nu_{\theta^\star}(\xx_t^-, t, c) = \E[\xx_1 - \xx_0^- \mid \xx_t^-, t, c]$.

Therefore, DiffusionNFT corresponds to the forward-KL branch of \eqref{eq:generalized_obj} with
\[
q_\theta^+ = \pi^+_\theta,
\qquad
q_\theta^- = \pi^-_\theta,
\qquad
\alpha = p_{\pi_{\mathrm{old}}}(o=1 \mid c),
\qquad
\gamma = p_{\pi_{\mathrm{old}}}(o=0 \mid c).
\]

\subsection{Connection to AWM}
\label{app:awm}
AWM~\citep{xue2025advantage} aims to optimize the following objective:
\begin{align}
    \label{eq:awm_objective}
    \max_\theta \E_{c,\, \xx_0 \sim \pi_\theta(\cdot \mid c)} [r(\xx_0,c)]
\end{align}
From \eqref{eq:pw-practical-decomposition}, we have:
\begin{align}
    -\log \pi_\theta(x_0 \mid c)
    &=
    H(q_\varepsilon^{x_0})
    +
    \E_{\substack{t,\, \epsilon}}[w(t)\|v_\theta(\xx_t,t,c)-(\xx_1-x_0)\|^2\bigr]
    +
    \mathcal G_{\varepsilon,w}(\theta;x_0)
    +
    \mathcal B_\varepsilon(\theta;x_0)
\end{align}

We set $\bar{\theta}$ as the stop gradient of $\theta$, we have:
\begin{align*}
    \log\frac{\pi_\theta(\xx_0 \mid c)}{\pi_{\bar{\theta}}(\xx_0 \mid c)}
    &= -\E_{t,\epsilon}[w(t)\|v_\theta(\xx_t,t,c)-(\xx_1-\xx_0)\|^2] \\
    &\quad +\E_{t,\epsilon}[w(t)\|v_{\bar{\theta}}(\xx_t,t,c)-(\xx_1-\xx_0)\|^2]\\
    &:= \Delta \mathcal{L}_{\mathrm{FM}}
\end{align*}
Therefore, the objective \eqref{eq:awm_objective} can be rewritten as
\begin{align*}
    &\max_\theta \E_{c,\, \xx_0 \sim \pi_\theta(\cdot \mid c)} [r(\xx_0,c)] \\
    &= \E_{c,\, \xx_0 \sim \pi_{\bar{\theta}}(\cdot \mid c)} \left[\frac{\pi_\theta(\xx_0 \mid c)}{\pi_{\bar{\theta}}(\xx_0 \mid c)}r(\xx_0,c)\right] \\
    &= \E_{c,\, \xx_0 \sim \pi_{\bar{\theta}}(\cdot \mid c)} \left[\exp(\Delta \mathcal{L}_{\mathrm{FM}})r(\xx_0,c)\right] 
\end{align*}

Therefore, AWM is optimized over forward process, while its derivation of $q^+_\theta, q^-_\theta, \alpha, \gamma$ is the same as FlowGRPO. Thus,
\[
q_\theta^+ = q_\theta^- = \pi_\theta,
\qquad
\alpha = 1,
\qquad
\gamma = -1,
\]

Because $\Delta \mathcal{L}_{\mathrm{FM}} = 0$ due to $\bar{\theta}$ is the stop gradient of $\theta$, we know $\exp(x) = 1 + x$ when $x$ is small. Therefore, we can simplify the objective as
\begin{align*}
    &\max_\theta \E_{c,\, \xx_0 \sim \pi_\theta(\cdot \mid c)} [r(\xx_0,c)] \\
    &= \E_{c,\, \xx_0 \sim \pi_{\bar{\theta}}(\cdot \mid c)} \left[\exp(\Delta \mathcal{L}_{\mathrm{FM}})r(\xx_0,c)\right] \\
    &= \E_{c,\, \xx_0 \sim \pi_{\bar{\theta}}(\cdot \mid c)} \left[(1 + \Delta \mathcal{L}_{\mathrm{FM}})r(\xx_0,c)\right] \\
    &= \E_{c,\, \xx_0 \sim \pi_{\bar{\theta}}(\cdot \mid c)} \left[r(\xx_0,c) + \Delta \mathcal{L}_{\mathrm{FM}}r(\xx_0,c)\right] \\
    &= -\E_{c,\, \xx_0 \sim \pi_{\bar{\theta}}(\cdot \mid c)} \left[r(\xx_0,c)\E_{t, \epsilon}[w(t)\|v_\theta(\xx_t,t,c)-(\xx_1-\xx_0)\|^2\bigr]\right] + C \\
    &= -\E_{c, \xx_0 \sim \pi_{\bar{\theta}}(\cdot \mid c),t, \epsilon} \left[w(t)r(\xx_0,c)\|v_\theta(\xx_t,t,c)-(\xx_1-\xx_0)\|^2\right] + C

\end{align*}

\subsection{Connection to RAM}
\label{app:ram}
First, let's remind the objective of RAM~\citep{bergmeister2026reinforce}:
\begin{align*}
    \mathcal{L}_{\mathrm{RAM}}(\theta) = \E_{c, t}[\| v_\theta(\xx_t,t,c) - sg(v_{\mathrm{ref}}(\xx_t,t,c)+r(\xx_0, c)((\xx_1-\xx_0) - v_\theta(\xx_t,t,c))) \|^2]
\end{align*}
where $sg$ is the stop gradient operator.

Let us define a new objective:
\begin{align*}
    \mathcal{L}_{2}(\theta) = \E_{c, t}[r(\xx_0, c)\| v_\theta(\xx_t,t,c) - (\xx_1-\xx_0) \|^2] + \E_{c, t}[\| v_\theta(\xx_t,t,c) - v_{\mathrm{ref}}(\xx_t,t,c) \|^2]
\end{align*}
It is easy to check that $\nabla_\theta \mathcal{L}_{2}(\theta) = \nabla_\theta \mathcal{L}_{\mathrm{RAM}}(\theta)$. 

Therefore, RAM can be considered as AWM objective plus a reference term. 

Thus, it is also optimized over forward process, while its derivation of $q^+_\theta, q^-_\theta, \alpha, \gamma$ is the same as FlowGRPO. Thus,
\[
q_\theta^+ = q_\theta^- = \pi_\theta,
\qquad
\alpha = 1,
\qquad
\gamma = -1,
\]

\subsection{Connection to SFT}
\label{app:sft}

The reference distribution decomposes as a mixture of the positive and negative posteriors.
\begin{lemma}
    \label{lem:sft_lemma}
    For every context $c$,
    \[
    \pi_{\mathrm{ref}}(\xx_0 \mid c)
    =
    p_{\pi_{\mathrm{ref}}}(o=1 \mid c)\,\pi^+(\xx_0 \mid c)
    +
    p_{\pi_{\mathrm{ref}}}(o=0 \mid c)\,\pi^-(\xx_0 \mid c).
    \]
\end{lemma}

\begin{proof}
Multiplying the definitions in \eqref{eq:posteriors} by the corresponding normalizing constants and summing the two identities gives
\begin{align*}
&p_{\pi_{\mathrm{ref}}}(o=1 \mid c)\,\pi^+(\xx_0 \mid c)
+ p_{\pi_{\mathrm{ref}}}(o=0 \mid c)\,\pi^-(\xx_0 \mid c) \\
&=
\pi_{\mathrm{ref}}(\xx_0 \mid c)\,[p(o=1 \mid \xx_0,c) + p(o=0 \mid \xx_0,c)] \\
&=
\pi_{\mathrm{ref}}(\xx_0 \mid c).
\end{align*}
\end{proof}
Standard SFT maximizes the log-likelihood of samples from $\pi_{\mathrm{ref}}$. 
By \thmref{thm:flow-matching-upper-bound-uniform-regularity}, we know that
when $\E_{\xx_0 \sim \pi_{\mathrm{ref}}(\cdot \mid c)}\bigl[\|v_\theta(\xx_t,t)-(\xx_1-\xx_0)\|^2\bigr]$ achieves the minimum, $\E_{\xx_0 \sim \pi_{\mathrm{ref}}(\cdot \mid c)}[-\log \pi_\theta(\xx_0 \mid c)]$ also achieves the minimum. Therefore, we have:
By \lemref{lem:sft_lemma},
\begin{align*}
&=\argmin_\theta \E_{\xx_0 \sim \pi_{\mathrm{ref}}(\cdot \mid c)}[-\log \pi_\theta(\xx_0 \mid c)] \\
&=
\int \pi_{\mathrm{ref}}(\xx_0 \mid c)[-\log \pi_\theta(\xx_0 \mid c)] d\xx_0 \\
&= 
\int (p_{\pi_{\mathrm{ref}}}(o=1 \mid c)\,\pi^+(\xx_0 \mid c) + p_{\pi_{\mathrm{ref}}}(o=0 \mid c)\,\pi^-(\xx_0 \mid c))[-\log \pi_\theta(\xx_0 \mid c)] d\xx_0 \\
&=
p_{\pi_{\mathrm{ref}}}(o=1 \mid c)\, \mathcal{D}_{KL}(\pi^+ \| \pi_\theta)
+ p_{\pi_{\mathrm{ref}}}(o=0 \mid c)\, \mathcal{D}_{KL}(\pi^- \| \pi_\theta)
\end{align*}
Thus SFT corresponds to the forward-KL branch of \eqref{eq:generalized_obj} with
\[
q_\theta^+ = q_\theta^- = \pi_\theta,
\qquad
\alpha = p_{\pi_{\mathrm{ref}}}(o=1 \mid c),
\qquad
\gamma = p_{\pi_{\mathrm{ref}}}(o=0 \mid c).
\]

\subsection{Connection to RFT}
\label{app:rft}

RFT keeps only preferred samples. By \thmref{thm:flow-matching-upper-bound-uniform-regularity}, we know that
when $\E_{\xx_0 \sim \pi_{\mathrm{ref}}(\cdot \mid c)}\bigl[\|v_\theta(\xx_t,t)-(\xx_1-\xx_0)\|^2\bigr]$ achieves the minimum, $\E_{\xx_0 \sim \pi_{\mathrm{ref}}(\cdot \mid c)}[-\log \pi_\theta(\xx_0 \mid c)]$ also achieves the minimum. Therefore, we have:
\begin{align*}
    &=\arg\min_\theta \E_{\xx_0 \sim \pi^+(\cdot \mid c)}[-\log \pi_\theta(\xx_0 \mid c)] \\
    &= \arg\min_\theta \mathcal{D}_{KL}(\pi^+ \| \pi_\theta)
\end{align*}
Therefore RFT corresponds to the forward-KL branch of \eqref{eq:generalized_obj} with
\[
q_\theta^+ = \pi_\theta,
\qquad
\alpha = 1,
\qquad
\gamma = 0.
\]

\subsection{Connection to FlowDPO}
\label{app:flowdpo}

The connection to FlowDPO~\cite{liu2025improving} is heuristic rather than formal. Starting from the KL-regularized RLHF objective, DPO uses the reward parameterization
\begin{equation}
    \label{eq:flowdpo_r}
    r(\xx_0, c) = \eta \log \frac{\pi_\theta(\xx_0 \mid c)}{\pi_{\text{ref}}(\xx_0 \mid c)} + \eta \log Z(c),
\end{equation}
which, when substituted into the Bradley--Terry preference likelihood, yields the usual DPO objective
\[
\mathcal{L}_{\mathrm{DPO}}(\theta)
=
\E_{c,\xx_0^w,\xx_0^l}
\left[
-\log \sigma \left(
\eta \log \frac{\pi_\theta(\xx_0^w \mid c)}{\pi_{\mathrm{ref}}(\xx_0^w \mid c)}
- \eta \log \frac{\pi_\theta(\xx_0^l \mid c)}{\pi_{\mathrm{ref}}(\xx_0^l \mid c)}
\right)
\right]
\]
From \eqref{eq:pw-practical-decomposition}, we have:
\begin{align*}
    -\log \pi_\theta(x_0 \mid c)
    &=
    H(q_\varepsilon^{x_0})
    +
    \E_{\substack{t,\, \epsilon}}[w(t)\|v_\theta(\xx_t,t,c)-(\xx_1-x_0)\|^2\bigr]
    +
    \mathcal G_{\varepsilon,w}(\theta;x_0)
    +
    \mathcal B_\varepsilon(\theta;x_0)
\end{align*}
Therefore, we have:
\begin{align}
    \log\frac{\pi_\theta(\xx_0 \mid c)}{\pi_{\mathrm{ref}}(\xx_0 \mid c)}
    = \Delta \mathcal{L}_{\mathrm{FM}}(\xx_0) + \Delta \mathcal{G}_{\varepsilon,w}(\xx_0) + \Delta \mathcal{B}_\varepsilon(\xx_0)
\end{align}
where
\begin{align*}
    \Delta \mathcal{L}_{\mathrm{FM}}(\xx_0)
    &:=
    \E_{\substack{t,\, \epsilon}}[-w(t)\|v_\theta(\xx_t,t,c)-(\xx_1-\xx_0)\|^2 + w(t)\|v_{\mathrm{ref}}(\xx_t,t,c)-(\xx_1-\xx_0)\|^2\bigr]\\
    \Delta \mathcal{G}_{\varepsilon,w}(\xx_0)
    &:=
    - \mathcal G_{\varepsilon,w}(\theta;\xx_0) + \mathcal G_{\varepsilon,w}(\mathrm{ref};\xx_0)\\
    \Delta \mathcal{B}_\varepsilon(\xx_0)
    &:=
    - \mathcal B_\varepsilon(\theta;\xx_0) + \mathcal B_\varepsilon(\mathrm{ref};\xx_0)
\end{align*}

Therefore, we can simplify the DPO objective in the following:
\begin{align*}
    \mathcal{L}_{\mathrm{DPO}}(\theta)
    &= \E_{c,\xx_0^w,\xx_0^l}
    \big[
    -\log \sigma \big(
    \eta \{
        \Delta \mathcal{L}_{\mathrm{FM}}(\xx^w_0) + \Delta \mathcal{G}_{\varepsilon,w}(\xx^w_0) + \Delta \mathcal{B}_\varepsilon(\xx^w_0)\\
    &\qquad\qquad\qquad\qquad\qquad -
        \Delta \mathcal{L}_{\mathrm{FM}}(\xx^l_0) - \Delta \mathcal{G}_{\varepsilon,w}(\xx^l_0) - \Delta \mathcal{B}_\varepsilon(\xx^l_0)
    \}
    \big)\\
    &= \E_{c,\xx_0^w,\xx_0^l}
    \big[
    -\log \sigma \big(
    \eta(\Delta \mathcal{L}_{\mathrm{FM}}(\xx^w_0) -
    \Delta \mathcal{L}_{\mathrm{FM}}(\xx^l_0))\\
    &\qquad\qquad\qquad\qquad\qquad + 
    \eta(\Delta \mathcal{G}_{\varepsilon,w}(\xx^w_0) + \Delta \mathcal{B}_\varepsilon(\xx^w_0) 
    - \Delta \mathcal{G}_{\varepsilon,w}(\xx^l_0) - \Delta \mathcal{B}_\varepsilon(\xx^l_0))
    \}\big)\big] \\
    & \text{because $-\log \sigma$ is convex and then by Jensen's inequality} \\
    &\le  \E_{c,\xx_0^w,\xx_0^l}
    \big[
    -\frac{1}{2}\log \sigma \big(
    2\eta(\Delta \mathcal{L}_{\mathrm{FM}}(\xx^w_0) -
    \Delta \mathcal{L}_{\mathrm{FM}}(\xx^l_0))
    \big)\\
    &\qquad\qquad\qquad-\frac{1}{2} \log \sigma \big(2\eta(\Delta \mathcal{G}_{\varepsilon,w}(\xx^w_0) + \Delta \mathcal{B}_\varepsilon(\xx^w_0) 
    - \Delta \mathcal{G}_{\varepsilon,w}(\xx^l_0) - \Delta \mathcal{B}_\varepsilon(\xx^l_0))
    \big)
    \big]\\
    &\le \frac{1}{2}\mathcal{L}_{\mathrm{FlowDPO}}(\theta) \\
    &\qquad+\E_{c,\xx_0^w,\xx_0^l}
    \big[-\frac{1}{2} \log \sigma \big(2\eta(\Delta \mathcal{G}_{\varepsilon,w}(\xx^w_0) + \Delta \mathcal{B}_\varepsilon(\xx^w_0) 
    - \Delta \mathcal{G}_{\varepsilon,w}(\xx^l_0) - \Delta \mathcal{B}_\varepsilon(\xx^l_0))
    \big)\big]
\end{align*}
where
\[
    \delta_\theta(\xx_0)
    :=\|v_\theta(\xx_t,t,c)-(\xx_1-\xx_0)\|^2
      -\|v_{\mathrm{ref}}(\xx_t,t,c)-(\xx_1-\xx_0)\|^2 .
\]
\begin{align*}
    \mathcal{L}_{\mathrm{FlowDPO}}(\theta) = \E_{c,\xx_0^w,\xx_0^l,t, \epsilon}
    \big[-\log \sigma \big(2\eta w(t)
    [\delta_\theta(\xx_0^l)-\delta_\theta(\xx_0^w)]\big)\big].
\end{align*}
Therefore, the FlowDPO objective plus a residual term is an upper bound of the DPO objective.

\newpage
\section{Experimental Details}
\label{app:details}

\subsection{Experimental Details of Sec. \ref{subsec:main_results}}
\label{app:experimental_details}

\begin{algorithm}[t]
    \caption{Flow Contrastive Preference Optimization (FlowCPO)}
    \label{alg:flowcpo}
    \small
    \begin{algorithmic}[1]
    \Require Pretrained $v_{\mathrm{ref}}$, dataset $\mathcal{D} =\{(c, \xx^w_0, \xx^l_0)\}$, flow interpolation coefficient $\beta$, negative regularization weight $\lambda$, EMA coefficient $\eta$, learning rate $\iota$.
    \State \textbf{Init:} $v_\theta \leftarrow v_{\mathrm{ref}}$, $v_{\mathrm{old}} \leftarrow v_{\mathrm{ref}}$.
    \For{each training iteration}
        \State Sample batch $(c,\xx_0^w,\xx_0^l) \sim \mathcal{D}$, $t \sim \mathcal{U}(0,1)$, $\epsilon \sim \mathcal{N}(0,I)$.
        \State \textbf{Forward:} $\xx_t^{\{w,l\}} = (1 - t) \cdot \xx_0^{\{w,l\}} + t \cdot \epsilon$; \quad Target $u_t^{\{w,l\}} = \epsilon - \xx_0^{\{w,l\}}$.
        \State \textbf{Flow Mixing:}
        \Statex \hskip1.5em $\mu_\theta = (1 - \beta) v_{\mathrm{old}} + \beta v_{\theta}, \quad \nu_\theta = (1 + \beta) v_{\mathrm{old}} - \beta v_{\theta}$.
        \State \textbf{Loss Update:}
        \Statex \hskip1.5em $\mathcal{L} = \|\mu_{\theta}(\xx_t^w, c, t) - u_t(\xx_t^w|\xx_0^w)\|_2^2 + \lambda \|\nu_{\theta}(\xx_t^l, c, t) - u_t(\xx_t^l|\xx_0^l)\|_2^2$
        \State $\theta \leftarrow \theta - \iota \cdot \nabla_\theta \mathcal{L}$.
        \State $v_{\mathrm{old}} \leftarrow EMA_{\eta}(v_{\theta}, v_{\mathrm{old}})$.
    \EndFor
    \State \Return $v_\theta$
    \end{algorithmic}
\end{algorithm}

\paragraph{In-Domain Offline Preference Data.} For the \emph{in-domain} offline training setting, we build a static preference dataset from a frozen copy of the pretrained reference model \texttt{Stable Diffusion 3.5 Medium (SD3.5-M)}. We sample prompts from the three target sources: the GenEval prompt set, the OCR prompt set, and the general-preference prompt set. The resulting in-domain prompt pools contain 50,000 training prompts for GenEval, 19,653 training prompts for OCR, and 25,432 training prompts for general preference. Their held-out evaluation splits contain 2212, 1018, and 2048 prompts, respectively. For each prompt, we generate 16 candidate images with the frozen reference model and convert them into a preferred/dispreferred pair using a fixed offline pipeline: the highest-scoring sample becomes the preferred and the lowest-scoring sample becomes the dispreferred. The GenEval and OCR training data use single-metric filtering aligned with their target benchmark. The general-preference training data uses an equal-weight multi-reward score with \texttt{PickScore}:\texttt{HPS v2.1}:\texttt{CLIP Score} = 1:1:1 during preferred/dispreferred construction. Because the generator parameters stay fixed during data construction, the resulting training distribution remains stationary throughout optimization. These preferred/dispreferred pairs should be read as empirical source distributions induced by the fixed best-of-16 pipeline, rather than as exact samples from the idealized posteriors $\pi^+$ and $\pi^-$ introduced in the theory.

All images used for data construction and evaluation are generated at a resolution of 512$\times$512 with 40 inference steps using the ODE sampler. Unless otherwise noted, we use a guidance scale of 4.5 during dataset construction. We leave negative prompts blank. The VAE and text encoder are frozen, and training prompts do not overlap with validation prompts or evaluation benchmarks. The complete reported experiment suite used approximately 288 aggregate A100 GPU-hours.

\paragraph{Out-of-Domain Offline Preference Data.} For the \emph{out-of-domain} offline training setting, we use Open Image Preferences v1 Results from the Data Is Better Together collection on Hugging Face~\cite{open_image_preferences_v1_results}. This dataset contains roughly 10K text-to-image preference pairs with community annotations, and the images are generated by open models including FLUX.1-Dev and SD3.5-Large. We use it to evaluate the out-of-domain regime because the training samples are not generated by the reference model used in our fine-tuning experiments.

\paragraph{Training Configuration.} We fine-tune \texttt{SD3.5-M} with LoRA on 8$\times$A100 GPUs in fp16 precision. Unless otherwise specified, the LoRA hyperparameters are $r = 32$ and $\alpha = 64$. We optimize with AdamW using a learning rate of 3e-4, cosine decay, and weight decay of 1e-4. The global batch size is 16. The EMA decay $\eta$ is 0.99, the default value of $\lambda$ is 1, and reported checkpoints are selected with the corresponding validation score for each training target.

\paragraph{Checkpoint Selection.} All fine-tuning runs are monitored for 2K training steps. For the in-domain specialists, we report the checkpoint with the best task-aligned validation score within that 2K-step budget: the GenEval model is selected by GenEval validation score, the OCR model by OCR validation score, and the general-preference model by the equal-weight composite validation score over \texttt{PickScore}, \texttt{HPS v2.1}, and \texttt{CLIP Score}. For the OOD setting, where one model is shared across all downstream evaluations, we select the checkpoint with the highest validation \texttt{PickScore} within the same 2K-step budget. We apply the same protocol to FlowCPO and the offline baselines in the corresponding setting.

\subsection{Experimental Results on Optimization Targets}
\label{app:optimization_targets}

We report results for three capabilities: semantic alignment, typographic generation, and general preference alignment. \tabref{tab:geneval_finetuning}, \tabref{tab:ocr_finetuning}, and \tabref{tab:multi_metric_finetuning} correspond to \emph{in-domain} fine-tuning with domain-specific reward filtering, whereas \tabref{tab:multi_metric_finetuning_ood} reports the \emph{out-of-domain} setting. The first three tables therefore summarize separate in-domain specialist models, while the OOD table evaluates single models trained once on out-of-domain data. Unless otherwise noted, each fine-tuned \texttt{SD3.5-M} entry is the mean over five independent training and evaluation runs with different random seeds, and the uncertainty is the sample standard deviation of the five run-level metric values. Pretrained baselines are shown as single evaluations for reference.

\tabref{tab:geneval_finetuning} reports quantitative results on \texttt{GenEval}, with representative qualitative comparisons in \figref{fig:flowcpo_geneval}. \tabref{tab:ocr_finetuning} summarizes typographic generation performance measured by OCR accuracy, with examples in \figref{fig:flowcpo_ocr}. \tabref{tab:multi_metric_finetuning} presents the in-domain results for general preference alignment; because training data in that setting are filtered with \texttt{PickScore}, \texttt{CLIP Score}, and \texttt{HPS v2.1}, the remaining metrics are the more informative check of transfer beyond the filtering pipeline. \tabref{tab:multi_metric_finetuning_ood} reports the out-of-domain setting, with qualitative examples in \figref{fig:flowcpo_geneval_ood}, \figref{fig:flowcpo_ocr_ood}, and \figref{fig:flowcpo_image_quality_ood}. Unless otherwise noted, we use FlowDPO temperature 100 and FlowCPO $\beta=0.5$ as the default settings.

\begin{table*}[h]
    \centering
    \caption{Quantitative comparison of fine-tuning SD3.5-M using different methods under the \emph{in-domain} offline training setting. The fine-tuning data are filtered solely by the \textbf{GenEval} reward function. We evaluate each model with several Classifier-Free Guidance (CFG) scales. The broad set of auxiliary metrics is included to monitor whether optimizing only for \texttt{GenEval} degrades other capabilities. Fine-tuned SD3.5-M variants are reported as mean $\pm$ standard deviation over 5 runs; pretrained baselines are listed as single evaluations. Within the SD3.5-M group, the best results are highlighted in \textbf{bold}, and the second best are \underline{underlined}.}
    \label{tab:geneval_finetuning}
    \resizebox{\textwidth}{!}{
    \begin{tabular}{l l >{\columncolor{gray!15}}c c c c c c c c}
    \toprule
    \textbf{Model} & \textbf{CFG} & \textbf{GenEval} & \textbf{OCR} & \textbf{PickScore} & \textbf{ClipScore} & \textbf{HPSv2.1} & \textbf{Aesthetic} & \textbf{ImgRwd} & \textbf{UniRwd} \\
    \midrule
    \multicolumn{10}{l}{\textit{Pretrained Model Baselines}} \\
    SD-XL & $-$ & 0.55 & 0.14 & 22.42 & 0.287 & 0.280 & 5.60 & 0.76 & 2.93 \\
    SD3.5-L & $-$ & 0.71 & 0.68 & 22.91 & 0.289 & 0.288 & 5.50 & 0.96 & 3.25 \\
    FLUX.1-Dev & $-$ & 0.66 & 0.59 & 22.84 & 0.295 & 0.274 & 5.71 & 0.96 & 3.27 \\
    \midrule
    \multicolumn{10}{l}{\textit{SD3.5-M Fine-Tuning (filtered by GenEval)}} \\
    \multirow{3}{*}{Base Model} 
    & $1.0$ & 0.24 & 0.12 & 20.51 & 0.237 & 0.204 & 5.13 & $-$0.58 & 2.02 \\
    & $3.0$ & 0.59 & 0.47 & 22.28 & 0.287 & 0.284 & 5.38 & 0.71 & 2.96 \\
    & $4.5$ & 0.63 & \textbf{0.59} & 22.34 & 0.285 & 0.279 & 5.36 & 0.85 & 3.03 \\
    \cmidrule(l){2-10} 
    \multirow{3}{*}{+ RFT~\cite{xiong2025minimalist,chen2025bridging} } 
    & $1.0$ & $0.5883 \pm 0.0100$ & $0.1421 \pm 0.0033$ & $21.6937 \pm 0.0125$ & $0.2735 \pm 0.0009$ & $0.2686 \pm 0.0005$ & $5.3183 \pm 0.0155$ & $0.4381 \pm 0.0179$ & $2.6175 \pm 0.0135$ \\
    & $3.0$ & $0.7423 \pm 0.0097$ & $0.5003 \pm 0.0070$ & \underline{$22.4661 \pm 0.0212$} & $0.2941 \pm 0.0006$ & \underline{$0.3010 \pm 0.0006$} & \underline{$5.3998 \pm 0.0092$} & $1.0200 \pm 0.0172$ & $3.1360 \pm 0.0145$ \\
    & $4.5$ & $0.7489 \pm 0.0045$ & $0.5562 \pm 0.0098$ & $\mathbf{22.4864 \pm 0.0099}$ & $0.2966 \pm 0.0005$ & $\mathbf{0.3031 \pm 0.0002}$ & $\mathbf{5.4070 \pm 0.0050}$ & $1.0808 \pm 0.0074$ & \underline{$3.1615 \pm 0.0080$} \\
    \cmidrule(l){2-10}
    \multirow{3}{*}{+ FlowDPO~\cite{liu2025improving}} 
    & $1.0$ & $0.5892 \pm 0.0046$ & $0.2327 \pm 0.0053$ & $21.4067 \pm 0.0324$ & $0.2696 \pm 0.0010$ & $0.2562 \pm 0.0007$ & $5.2817 \pm 0.0058$ & $0.4560 \pm 0.0225$ & $2.5870 \pm 0.0205$ \\
    & $3.0$ & $0.8118 \pm 0.0044$ & $0.4534 \pm 0.0075$ & $22.3095 \pm 0.0182$ & \underline{$0.2970 \pm 0.0006$} & $0.2938 \pm 0.0009$ & $5.3570 \pm 0.0112$ & \underline{$1.0811 \pm 0.0094$} & \underline{$3.1615 \pm 0.0200$} \\
    & $4.5$ & $0.8107 \pm 0.0046$ & $0.5067 \pm 0.0074$ & $22.3091 \pm 0.0138$ & $\mathbf{0.2979 \pm 0.0005}$ & $0.2963 \pm 0.0003$ & $5.3479 \pm 0.0052$ & $\mathbf{1.1256 \pm 0.0078}$ & $\mathbf{3.1680 \pm 0.0060}$ \\
    \cmidrule(l){2-10}
    \multirow{3}{*}{+ FlowCPO ($\beta = 0.5$, Ours)} 
    & $1.0$ & $0.7596 \pm 0.0106$ & $0.2535 \pm 0.0098$ & $21.8045 \pm 0.0157$ & $0.2804 \pm 0.0001$ & $0.2673 \pm 0.0008$ & $5.2487 \pm 0.0139$ & $0.6231 \pm 0.0160$ & $2.8040 \pm 0.0125$ \\
    & $3.0$ & $\mathbf{0.8415 \pm 0.0036}$ & $0.5203 \pm 0.0055$ & $22.1952 \pm 0.0184$ & $0.2957 \pm 0.0005$ & $0.2885 \pm 0.0007$ & $5.2662 \pm 0.0085$ & $0.9999 \pm 0.0076$ & $3.1170 \pm 0.0085$ \\
    & $4.5$ & \underline{$0.8161 \pm 0.0052$} & \underline{$0.5702 \pm 0.0082$} & $21.9438 \pm 0.0086$ & $0.2933 \pm 0.0005$ & $0.2807 \pm 0.0004$ & $5.1610 \pm 0.0062$ & $0.9184 \pm 0.0094$ & $3.0245 \pm 0.0150$ \\
    \bottomrule
    \end{tabular}
    }
\end{table*}

\begin{table*}[h]
    \centering
    \caption{Quantitative comparison of fine-tuning SD3.5-M using different methods (RFT, FlowDPO, FlowCPO) under the \emph{in-domain} offline training setting. The fine-tuning data are filtered solely by the \textbf{OCR} reward metric. We evaluate the models with several Classifier-Free Guidance (CFG) scales. The broad set of auxiliary metrics is included to monitor whether optimizing only for OCR degrades other generation capabilities. Fine-tuned SD3.5-M variants are reported as mean $\pm$ standard deviation over 5 runs; pretrained baselines are listed as single evaluations. Within the SD3.5-M fine-tuning group, the best results are highlighted in \textbf{bold}, and the second best are \underline{underlined}.}
    \label{tab:ocr_finetuning}
    \resizebox{\textwidth}{!}{
    \begin{tabular}{l l c >{\columncolor{gray!15}}c c c c c c c}
    \toprule
    \textbf{Model} & \textbf{CFG} & \textbf{GenEval} & \textbf{OCR} & \textbf{PickScore} & \textbf{ClipScore} & \textbf{HPSv2.1} & \textbf{Aesthetic} & \textbf{ImgRwd} & \textbf{UniRwd} \\
    \midrule
    \multicolumn{10}{l}{\textit{Pretrained Model Baselines}} \\
    SD-XL & $ - $ & 0.55 & 0.14 & 22.42 & 0.287 & 0.280 & 5.60 & 0.76 & 2.93 \\
    SD3.5-L & $ - $ & 0.71 & 0.68 & 22.91 & 0.289 & 0.288 & 5.50 & 0.96 & 3.25 \\
    FLUX.1-Dev & $ - $ & 0.66 & 0.59 & 22.84 & 0.295 & 0.274 & 5.71 & 0.96 & 3.27 \\
    \midrule
    \multicolumn{10}{l}{\textit{SD3.5-M Fine-Tuning (filtered by OCR)}} \\
    \multirow{3}{*}{Base Model} 
    & $1.0$ & 0.24 & 0.12 & 20.51 & 0.237 & 0.204 & 5.13 & $-$0.58 & 2.02 \\
    & $3.0$ & 0.59 & 0.47 & 22.28 & 0.287 & 0.284 & 5.38 & 0.71 & 2.96 \\
    & $4.5$ & 0.63 & 0.59 & 22.34 & 0.285 & 0.279 & 5.36 & 0.85 & 3.03 \\
    \cmidrule(l){2-10} 
    \multirow{3}{*}{+ RFT~\cite{xiong2025minimalist,chen2025bridging} } 
    & $1.0$ & $0.5127 \pm 0.0108$ & $0.3530 \pm 0.0057$ & $21.8831 \pm 0.0192$ & $0.2784 \pm 0.0007$ & $0.2741 \pm 0.0005$ & $5.3589 \pm 0.0096$ & $0.5248 \pm 0.0191$ & $2.7170 \pm 0.0160$ \\
    & $3.0$ & \underline{$0.6577 \pm 0.0106$} & $0.6981 \pm 0.0105$ & $\mathbf{22.5177 \pm 0.0194}$ & \underline{$0.2964 \pm 0.0008$} & \underline{$0.3008 \pm 0.0003$} & $\mathbf{5.3954 \pm 0.0043}$ & \underline{$1.0491 \pm 0.0072$} & $\mathbf{3.1630 \pm 0.0200}$ \\
    & $4.5$ & $\mathbf{0.6698 \pm 0.0101}$ & $0.7202 \pm 0.0083$ & \underline{$22.5098 \pm 0.0157$} & $\mathbf{0.2979 \pm 0.0009}$ & $\mathbf{0.3023 \pm 0.0005}$ & \underline{$5.3927 \pm 0.0068$} & $\mathbf{1.0816 \pm 0.0100}$ & \underline{$3.1605 \pm 0.0175$} \\
    \cmidrule(l){2-10}
    \multirow{3}{*}{+ FlowDPO~\cite{liu2025improving}} 
    & $1.0$ & $0.3140 \pm 0.0051$ & $0.5136 \pm 0.0050$ & $20.9171 \pm 0.0090$ & $0.2565 \pm 0.0008$ & $0.2219 \pm 0.0006$ & $5.1246 \pm 0.0127$ & $-0.3015 \pm 0.0273$ & $2.2580 \pm 0.0175$ \\
    & $3.0$ & $0.5848 \pm 0.0098$ & $0.7389 \pm 0.0096$ & $22.3210 \pm 0.0104$ & $0.2927 \pm 0.0004$ & $0.2873 \pm 0.0005$ & $5.3620 \pm 0.0025$ & $0.8559 \pm 0.0143$ & $3.0315 \pm 0.0220$ \\
    & $4.5$ & $0.6149 \pm 0.0153$ & $0.7476 \pm 0.0041$ & $22.4249 \pm 0.0144$ & $0.2960 \pm 0.0006$ & $0.2943 \pm 0.0005$ & $5.3795 \pm 0.0064$ & $0.9493 \pm 0.0182$ & $3.0955 \pm 0.0195$ \\
    \cmidrule(l){2-10}
    \multirow{3}{*}{+ FlowCPO ($\beta =0.5$, Ours)} 
    & $1.0$ & $0.4023 \pm 0.0062$ & $0.8349 \pm 0.0033$ & $21.2178 \pm 0.0185$ & $0.2689 \pm 0.0006$ & $0.2430 \pm 0.0011$ & $5.1770 \pm 0.0102$ & $0.0952 \pm 0.0182$ & $2.4955 \pm 0.0185$ \\
    & $3.0$ & $0.5962 \pm 0.0113$ & $\mathbf{0.8737 \pm 0.0034}$ & $22.3139 \pm 0.0090$ & $0.2930 \pm 0.0002$ & $0.2943 \pm 0.0005$ & $5.3506 \pm 0.0079$ & $0.9560 \pm 0.0115$ & $3.0760 \pm 0.0085$ \\
    & $4.5$ & $0.6124 \pm 0.0102$ & \underline{$0.8588 \pm 0.0078$} & $22.3656 \pm 0.0154$ & $0.2953 \pm 0.0008$ & $0.2982 \pm 0.0005$ & $5.3591 \pm 0.0048$ & $1.0438 \pm 0.0048$ & $3.1160 \pm 0.0035$ \\
    \bottomrule
    \end{tabular}
    }
\end{table*}

\begin{table*}[h]
    \centering
    \caption{Quantitative comparison of fine-tuning SD3.5-M using different methods (RFT, FlowDPO, and FlowCPO) under the \emph{in-domain} offline training setting. The fine-tuning data are filtered by a combination of multiple metrics (\textbf{PickScore}, \textbf{CLIP Score}, and \textbf{HPSv2.1}), which are highlighted in gray. We evaluate the models with several Classifier-Free Guidance (CFG) scales. Because the shaded metrics are directly coupled to the filtering pipeline, the remaining metrics are especially useful for assessing transfer beyond the optimization-aligned rewards. Fine-tuned SD3.5-M variants are reported as mean $\pm$ standard deviation over 5 runs; pretrained baselines are listed as single evaluations. Within the SD3.5-M fine-tuning group, the best results are highlighted in \textbf{bold}, and the second best are \underline{underlined}.}
    \label{tab:multi_metric_finetuning}
    \resizebox{\textwidth}{!}{
    \begin{tabular}{l l c c >{\columncolor{gray!15}}c >{\columncolor{gray!15}}c >{\columncolor{gray!15}}c c c c}
    \toprule
    \textbf{Model} & \textbf{CFG} & \textbf{GenEval} & \textbf{OCR} & \textbf{PickScore} & \textbf{ClipScore} & \textbf{HPSv2.1} & \textbf{Aesthetic} & \textbf{ImgRwd} & \textbf{UniRwd} \\
    \midrule
    \multicolumn{10}{l}{\textit{Pretrained Model Baselines}} \\
    SD-XL & $ - $ & 0.55 & 0.14 & 22.42 & 0.287 & 0.280 & 5.60 & 0.76 & 2.93 \\
    SD3.5-L & $ - $ & 0.71 & 0.68 & 22.91 & 0.289 & 0.288 & 5.50 & 0.96 & 3.25 \\
    FLUX.1-Dev & $ - $ & 0.66 & 0.59 & 22.84 & 0.295 & 0.274 & 5.71 & 0.96 & 3.27 \\
    \midrule
    \multicolumn{10}{l}{\textit{SD3.5-M Fine-Tuning}} \\
    \multirow{3}{*}{Base Model} 
    & $1.0$ & 0.24 & 0.12 & 20.51 & 0.237 & 0.204 & 5.13 & $-$0.58 & 2.02 \\
    & $3.0$ & 0.59 & 0.47 & 22.28 & 0.287 & 0.284 & 5.38 & 0.71 & 2.96 \\
    & $4.5$ & 0.63 & 0.59 & 22.34 & 0.285 & 0.279 & 5.36 & 0.85 & 3.03 \\
    \cmidrule(l){2-10} 
    \multirow{3}{*}{+ RFT~\cite{xiong2025minimalist,chen2025bridging}} 
    & $1.0$ & $0.5312 \pm 0.0111$ & $0.2351 \pm 0.0045$ & $21.9056 \pm 0.0124$ & $0.2793 \pm 0.0002$ & $0.2759 \pm 0.0008$ & $5.3507 \pm 0.0111$ & $0.5694 \pm 0.0213$ & $2.6590 \pm 0.0065$ \\
    & $3.0$ & \underline{$0.6894 \pm 0.0086$} & $0.5722 \pm 0.0045$ & $22.5975 \pm 0.0069$ & $0.2963 \pm 0.0007$ & $0.3041 \pm 0.0005$ & $5.4139 \pm 0.0093$ & $1.0616 \pm 0.0096$ & $3.1370 \pm 0.0100$ \\
    & $4.5$ & $\mathbf{0.6974 \pm 0.0056}$ & $\mathbf{0.6180 \pm 0.0097}$ & $22.5664 \pm 0.0104$ & $0.2972 \pm 0.0009$ & $0.3054 \pm 0.0005$ & $5.4207 \pm 0.0061$ & $1.1068 \pm 0.0063$ & $3.1520 \pm 0.0020$ \\
    \cmidrule(l){2-10}
    \multirow{3}{*}{+ FlowDPO~\cite{liu2025improving}} 
    & $1.0$ & $0.0363 \pm 0.0117$ & $0.2132 \pm 0.0061$ & $20.7158 \pm 0.0179$ & $0.2399 \pm 0.0006$ & $0.2205 \pm 0.0010$ & $5.2202 \pm 0.0113$ & $-0.4610 \pm 0.0222$ & $2.1075 \pm 0.0120$ \\
    & $3.0$ & $0.3903 \pm 0.0100$ & $0.5549 \pm 0.0072$ & $22.7637 \pm 0.0209$ & $0.2974 \pm 0.0008$ & $0.3014 \pm 0.0007$ & \underline{$5.5590 \pm 0.0089$} & $1.0403 \pm 0.0089$ & $3.0915 \pm 0.0195$ \\
    & $4.5$ & $0.4610 \pm 0.0137$ & $0.5791 \pm 0.0071$ & \underline{$22.8912 \pm 0.0192$} & $\mathbf{0.3010 \pm 0.0008}$ & \underline{$0.3107 \pm 0.0006$} & $5.5562 \pm 0.0107$ & $1.1543 \pm 0.0072$ & $3.1840 \pm 0.0190$ \\
    \cmidrule(l){2-10}
    \multirow{3}{*}{+ FlowCPO ($\beta =0.5$, Ours)} 
    & $1.0$ & $0.5325 \pm 0.0035$ & $0.3426 \pm 0.0053$ & $22.4805 \pm 0.0155$ & $0.2796 \pm 0.0007$ & $0.2900 \pm 0.0003$ & $\mathbf{5.6392 \pm 0.0100}$ & $0.9014 \pm 0.0139$ & $2.9300 \pm 0.0200$ \\
    & $3.0$ & $0.6487 \pm 0.0088$ & \underline{$0.5888 \pm 0.0096$} & $\mathbf{22.9432 \pm 0.0164}$ & $0.2994 \pm 0.0005$ & $\mathbf{0.3114 \pm 0.0004}$ & $5.5395 \pm 0.0065$ & $\mathbf{1.2480 \pm 0.0071}$ & $\mathbf{3.3020 \pm 0.0110}$ \\
    & $4.5$ & $0.6526 \pm 0.0079$ & $0.5762 \pm 0.0136$ & $22.6725 \pm 0.0088$ & \underline{$0.2997 \pm 0.0008$} & $0.3030 \pm 0.0004$ & $5.4569 \pm 0.0045$ & \underline{$1.2174 \pm 0.0116$} & \underline{$3.2775 \pm 0.0095$} \\
    \bottomrule
    \end{tabular}
    }
\end{table*}

\begin{table*}[h]
    \centering
    \caption{Quantitative comparison of fine-tuning SD3.5-M using different methods (RFT, FlowDPO, and FlowCPO) under the \emph{out-of-domain} offline training setting. The fine-tuning data come from the open-source Open Image Preferences v1 Results dataset and were generated by models other than the SD3.5-M reference policy. We evaluate the models with several Classifier-Free Guidance (CFG) scales. Fine-tuned SD3.5-M variants are reported as mean $\pm$ standard deviation over 5 runs; pretrained baselines are listed as single evaluations. Within the SD3.5-M fine-tuning group, the best results are highlighted in \textbf{bold}, and the second best are \underline{underlined}.}
    \label{tab:multi_metric_finetuning_ood}
    \resizebox{\textwidth}{!}{
    \begin{tabular}{l l c c c c c c c c}
    \toprule
    \textbf{Model} & \textbf{CFG} & \textbf{GenEval} & \textbf{OCR} & \textbf{PickScore} & \textbf{ClipScore} & \textbf{HPSv2.1} & \textbf{Aesthetic} & \textbf{ImgRwd} & \textbf{UniRwd} \\
    \midrule
    \multicolumn{10}{l}{\textit{Pretrained Model Baselines}} \\
    SD-XL & $ - $ & 0.55 & 0.14 & 22.42 & 0.287 & 0.280 & 5.60 & 0.76 & 2.93 \\
    SD3.5-L & $ - $ & 0.71 & 0.68 & 22.91 & 0.289 & 0.288 & 5.50 & 0.96 & 3.25 \\
    FLUX.1-Dev & $ - $ & 0.66 & 0.59 & 22.84 & 0.295 & 0.274 & 5.71 & 0.96 & 3.27 \\
    \midrule
    \multicolumn{10}{l}{\textit{SD3.5-M Fine-Tuning}} \\
    \multirow{3}{*}{Base Model} 
    & $1.0$ & 0.24 & 0.12 & 20.51 & 0.237 & 0.204 & 5.13 & $-$0.58 & 2.02 \\
    & $3.0$ & 0.59 & 0.47 & 22.28 & 0.287 & 0.284 & 5.38 & 0.71 & 2.96 \\
    & $4.5$ & 0.63 & 0.59 & 22.34 & 0.285 & 0.279 & 5.36 & 0.85 & 3.03 \\
    \cmidrule(l){2-10} 
    \multirow{3}{*}{+ RFT~\cite{xiong2025minimalist,chen2025bridging}} 
    & $1.0$ & $0.4398 \pm 0.0086$ & $0.1723 \pm 0.0061$ & $21.5913 \pm 0.0194$ & $0.2678 \pm 0.0016$ & $0.2611 \pm 0.0010$ & $5.4923 \pm 0.0129$ & $0.3110 \pm 0.0216$ & $2.5551 \pm 0.0138$ \\
    & $3.0$ & $0.6650 \pm 0.0076$ & $0.5310 \pm 0.0078$ & \underline{$22.6427 \pm 0.0108$} & \underline{$0.2974 \pm 0.0008$} & \underline{$0.3032 \pm 0.0002$} & $5.4780 \pm 0.0082$ & \underline{$1.0514 \pm 0.0134$} & $3.1766 \pm 0.0202$ \\
    & $4.5$ & $0.6845 \pm 0.0078$ & $\mathbf{0.5968 \pm 0.0059}$ & $\mathbf{22.6854 \pm 0.0063}$ & $\mathbf{0.2991 \pm 0.0007}$ & $\mathbf{0.3074 \pm 0.0003}$ & $5.4733 \pm 0.0074$ & $\mathbf{1.1228 \pm 0.0133}$ & $\mathbf{3.2299 \pm 0.0084}$ \\
    \cmidrule(l){2-10}
    \multirow{3}{*}{+ FlowDPO~\cite{liu2025improving}} 
    & $1.0$ & $0.2290 \pm 0.0084$ & $0.1158 \pm 0.0054$ & $20.8273 \pm 0.0155$ & $0.2425 \pm 0.0011$ & $0.2264 \pm 0.0006$ & $\mathbf{5.7158 \pm 0.0147}$ & $-0.3851 \pm 0.0285$ & $2.2693 \pm 0.0198$ \\
    & $3.0$ & $0.6065 \pm 0.0141$ & $0.4789 \pm 0.0053$ & $22.5399 \pm 0.0212$ & $0.2919 \pm 0.0007$ & $0.2917 \pm 0.0008$ & \underline{$5.5203 \pm 0.0032$} & $0.9063 \pm 0.0109$ & $3.0972 \pm 0.0170$ \\
    & $4.5$ & $0.6513 \pm 0.0132$ & $0.5406 \pm 0.0071$ & $22.6201 \pm 0.0165$ & $0.2946 \pm 0.0003$ & $0.2993 \pm 0.0003$ & $5.4927 \pm 0.0108$ & $0.9934 \pm 0.0164$ & $3.1694 \pm 0.0117$ \\
    \cmidrule(l){2-10}
    \multirow{3}{*}{+ FlowCPO ($\beta=1$, Ours)} 
    & $1.0$ & $0.4670 \pm 0.0110$ & $0.2351 \pm 0.0027$ & $21.8088 \pm 0.0198$ & $0.2762 \pm 0.0011$ & $0.2704 \pm 0.0003$ & $5.3546 \pm 0.0107$ & $0.5230 \pm 0.0140$ & $2.7723 \pm 0.0152$ \\
    & $3.0$ & \underline{$0.6958 \pm 0.0050$} & $0.5561 \pm 0.0085$ & $22.4609 \pm 0.0188$ & $0.2966 \pm 0.0003$ & $0.2952 \pm 0.0004$ & $5.4140 \pm 0.0029$ & $1.0166 \pm 0.0075$ & \underline{$3.2263 \pm 0.0100$} \\
    & $4.5$ & $\mathbf{0.6995 \pm 0.0060}$ & \underline{$0.5912 \pm 0.0069$} & $22.3570 \pm 0.0111$ & $0.2942 \pm 0.0007$ & $0.2942 \pm 0.0002$ & $5.3810 \pm 0.0056$ & $1.0159 \pm 0.0089$ & $3.2173 \pm 0.0184$ \\
    \cmidrule(l){2-10}
    \multirow{3}{*}{+ FlowCPO ($\beta=0.5$, Ours)} 
    & $1.0$ & $0.5706 \pm 0.0124$ & $0.2360 \pm 0.0047$ & $22.0231 \pm 0.0146$ & $0.2813 \pm 0.0006$ & $0.2752 \pm 0.0005$ & $5.3451 \pm 0.0106$ & $0.6294 \pm 0.0064$ & $2.8523 \pm 0.0200$ \\
    & $3.0$ & $0.6859 \pm 0.0099$ & $0.4919 \pm 0.0071$ & $22.3229 \pm 0.0064$ & $0.2959 \pm 0.0010$ & $0.2882 \pm 0.0003$ & $5.3718 \pm 0.0074$ & $0.9967 \pm 0.0062$ & $3.1464 \pm 0.0089$ \\
    & $4.5$ & $0.6815 \pm 0.0040$ & $0.4912 \pm 0.0054$ & $22.1181 \pm 0.0128$ & $0.2939 \pm 0.0009$ & $0.2842 \pm 0.0005$ & $5.3073 \pm 0.0069$ & $0.9452 \pm 0.0108$ & $3.1068 \pm 0.0117$ \\
    \bottomrule
    \end{tabular}
    }
\end{table*}

\begin{figure}[h]
    \centering
    \includegraphics[width=0.96\textwidth]{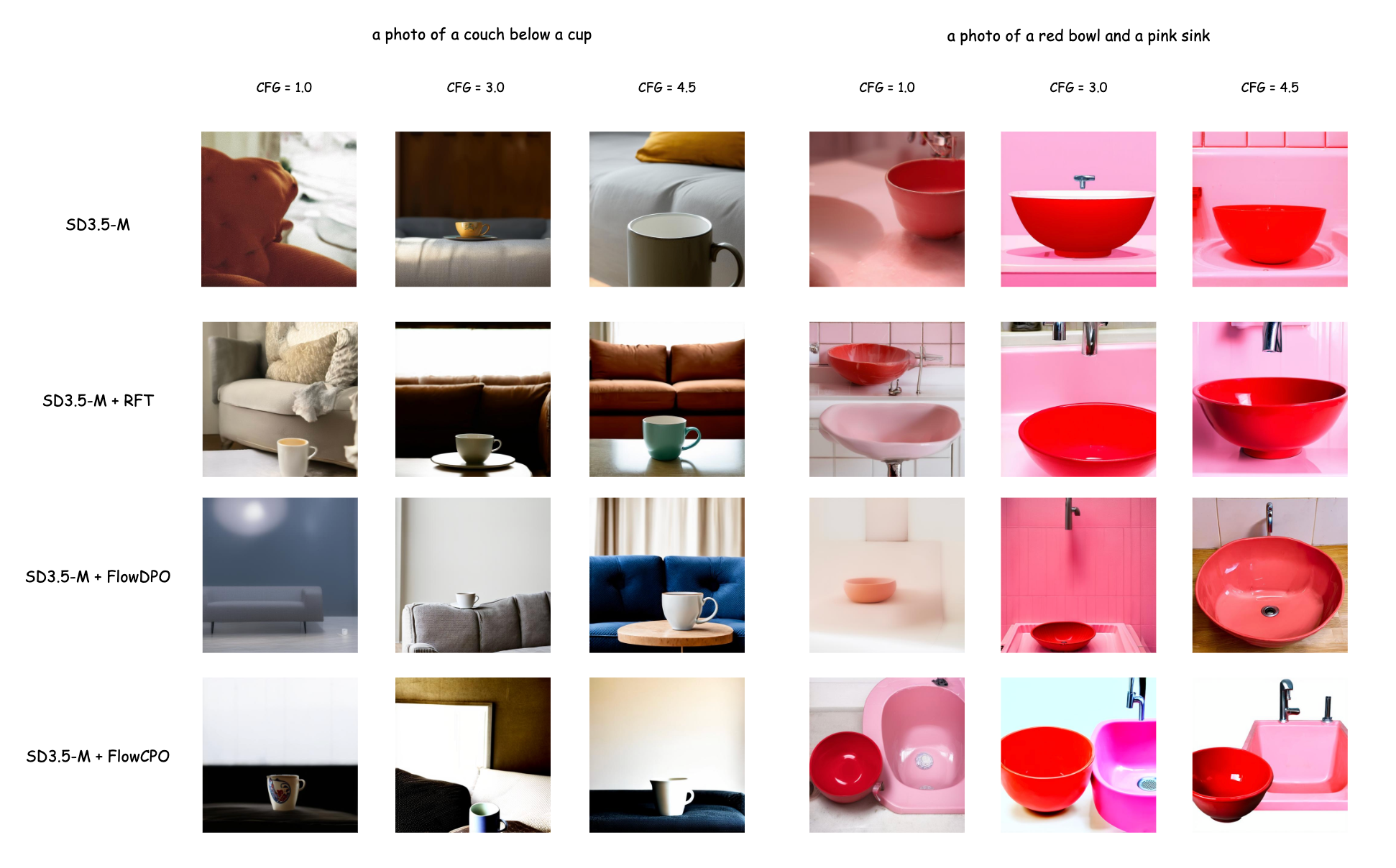}
    \caption{Qualitative comparison on the GenEval benchmark. All samples are generated by models trained with \emph{in-domain} preference data.}
    \label{fig:flowcpo_geneval}
\end{figure}

\begin{figure}[h]
    \centering
    \includegraphics[width=0.96\textwidth]{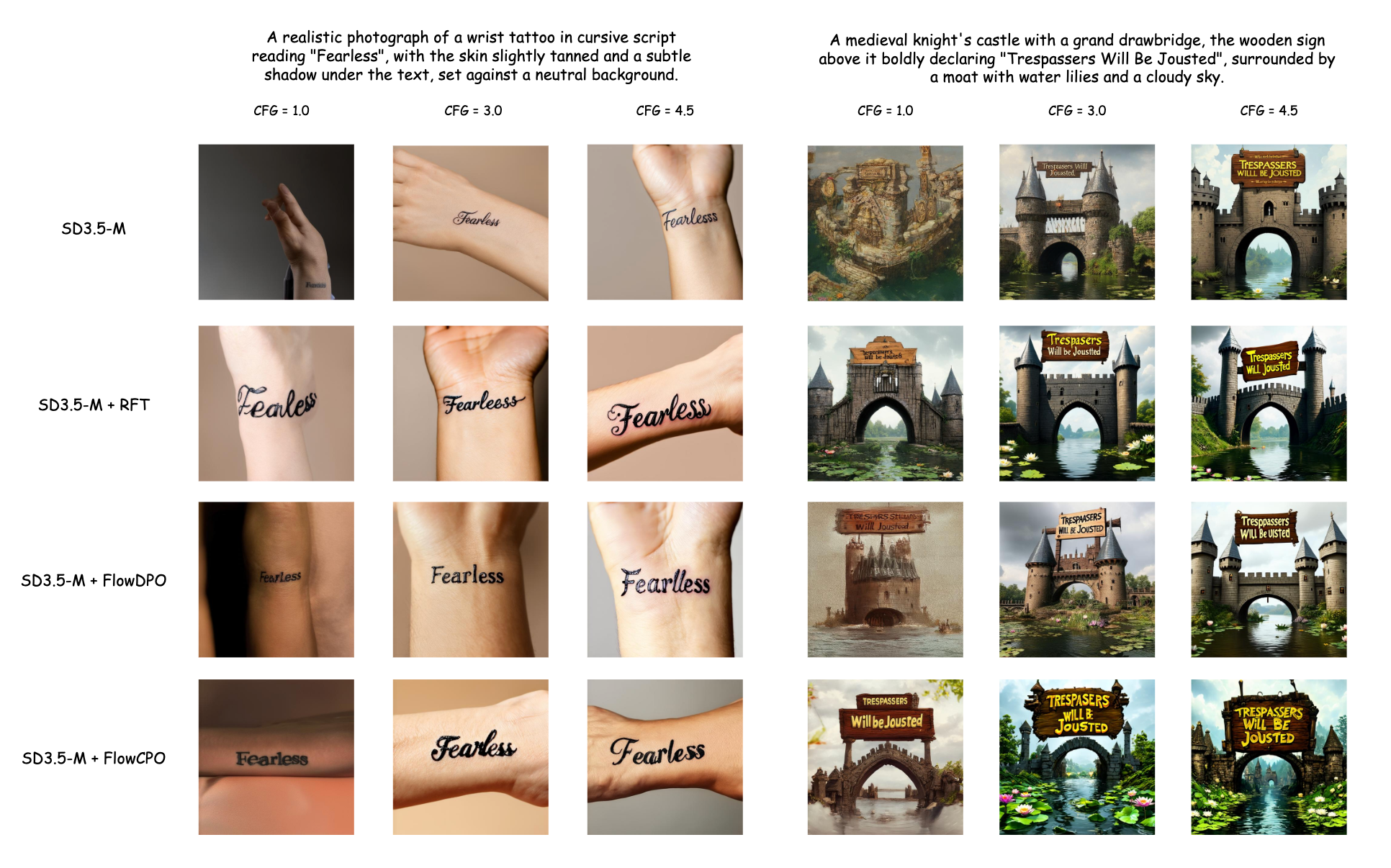}
    \caption{Qualitative comparison on the OCR benchmark. All samples are generated by models trained with \emph{in-domain} preference data.}
    \label{fig:flowcpo_ocr}
\end{figure}

\begin{figure}[h]
    \centering
    \includegraphics[width=0.84\textwidth]{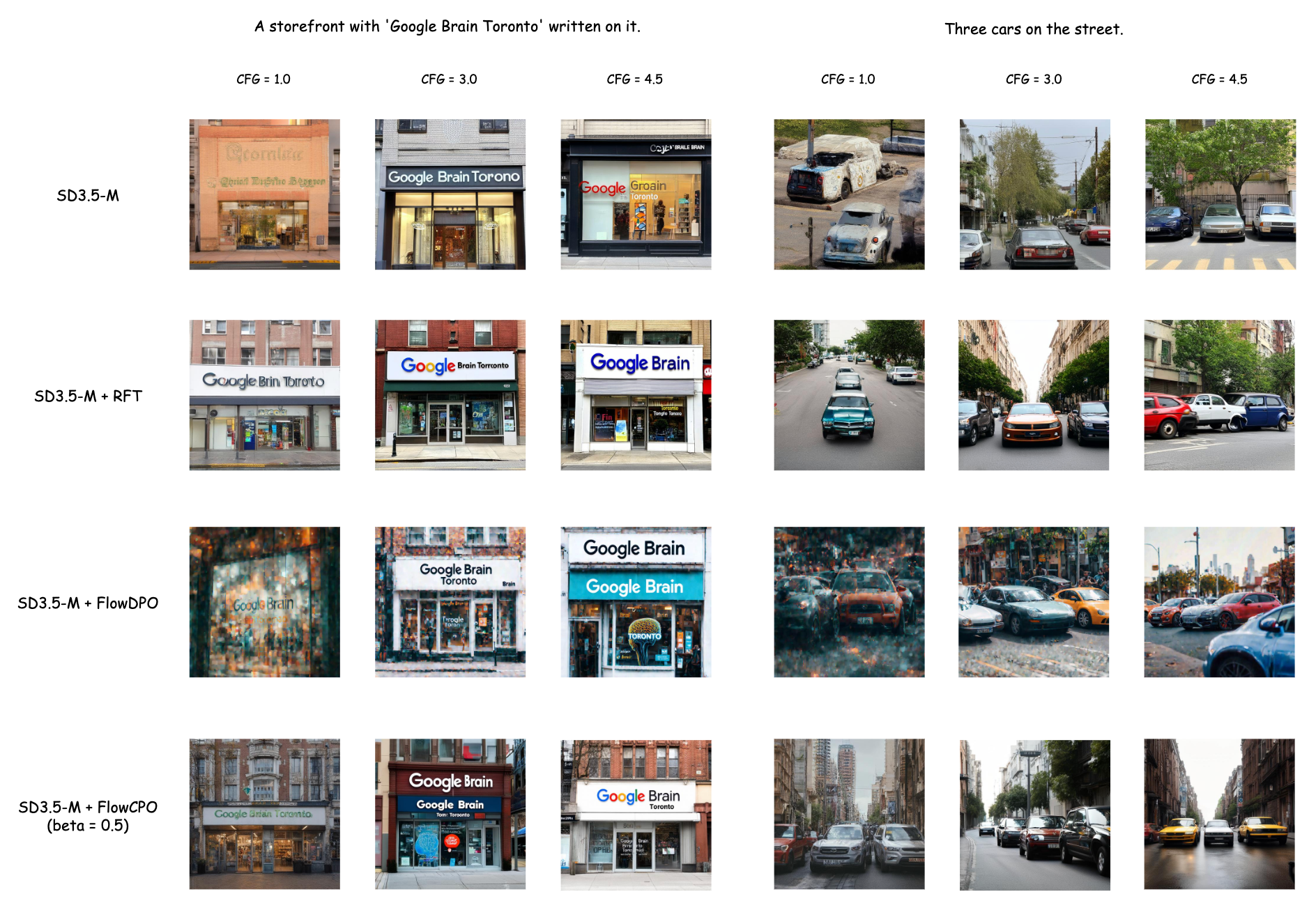}
    \caption{Qualitative comparison on the \texttt{DrawBench} benchmark. All samples are generated by models trained with \emph{in-domain} preference data.}
    \label{fig:flowcpo_image_quality}
\end{figure}

\begin{figure}[h]
    \centering
    \includegraphics[width=0.84\textwidth]{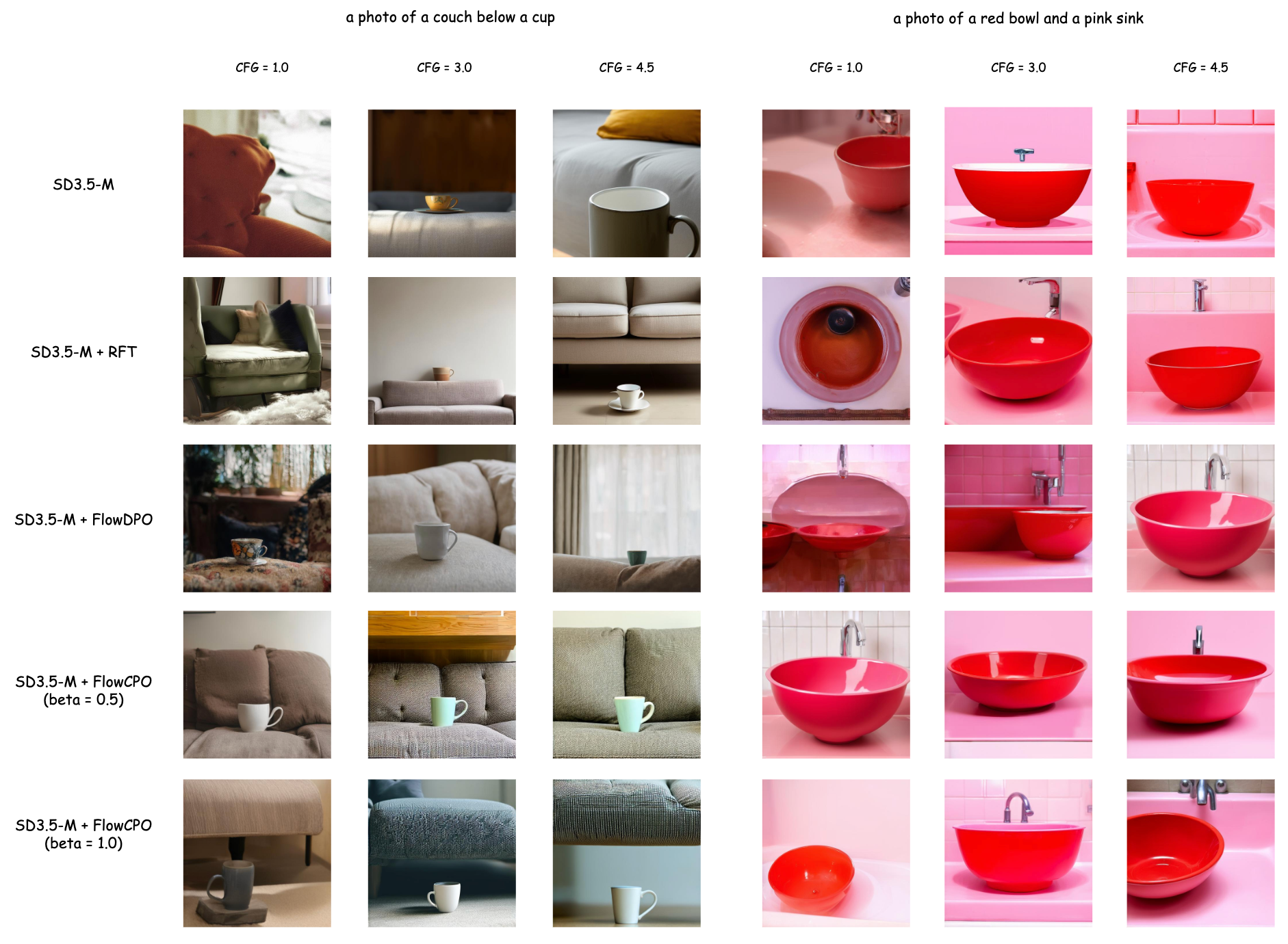}
    \caption{Qualitative comparison on the GenEval benchmark. All samples are generated by models trained with \emph{out-of-domain} preference data.}
    \label{fig:flowcpo_geneval_ood}
\end{figure}

\begin{figure}[h]
    \centering
    \includegraphics[width=0.82\textwidth]{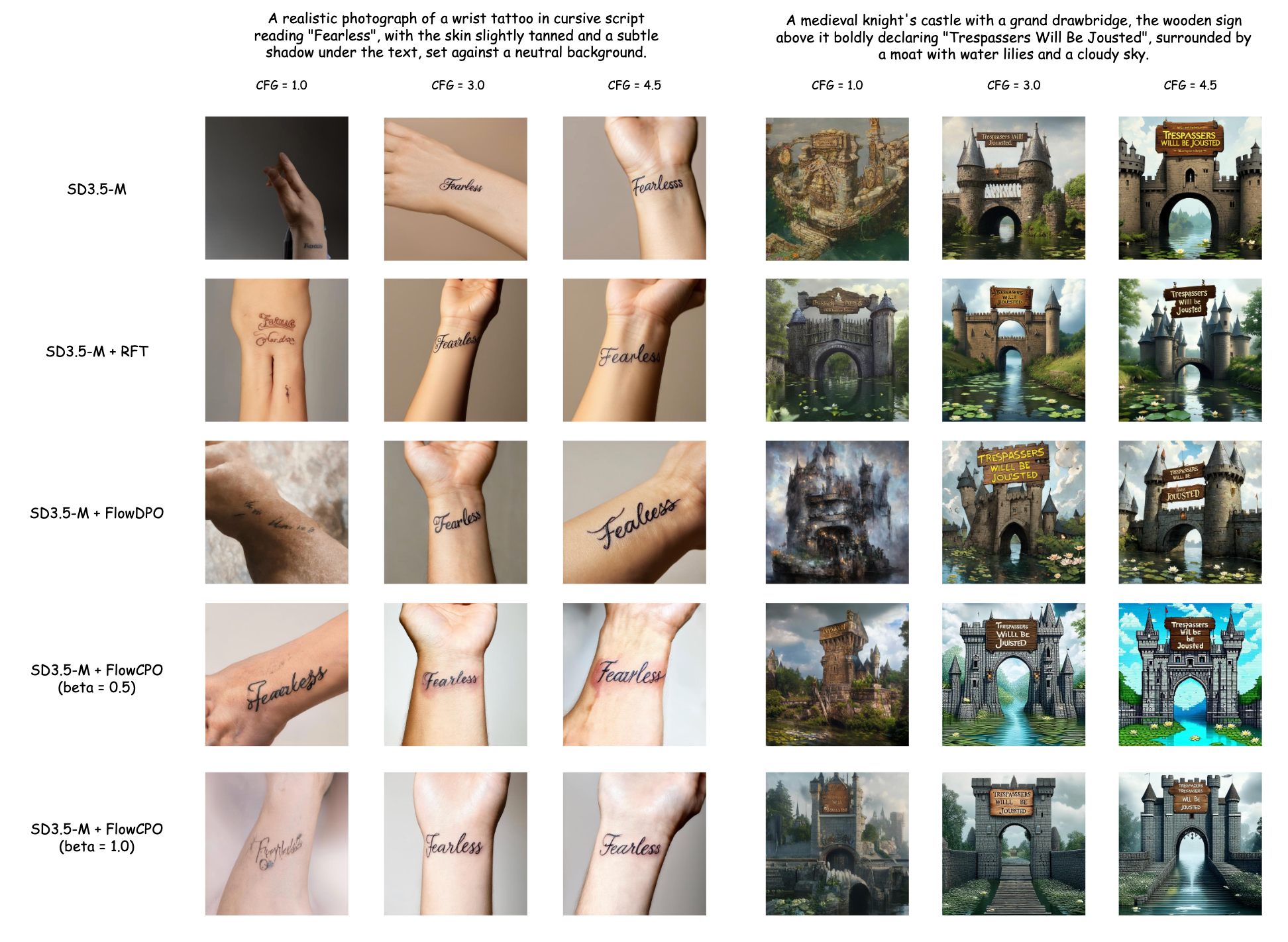}
    \caption{Qualitative comparison on the OCR benchmark. All samples are generated by models trained with \emph{out-of-domain} preference data.}
    \label{fig:flowcpo_ocr_ood}
\end{figure}

\begin{figure}[h]
    \centering
    \includegraphics[width=0.82\textwidth]{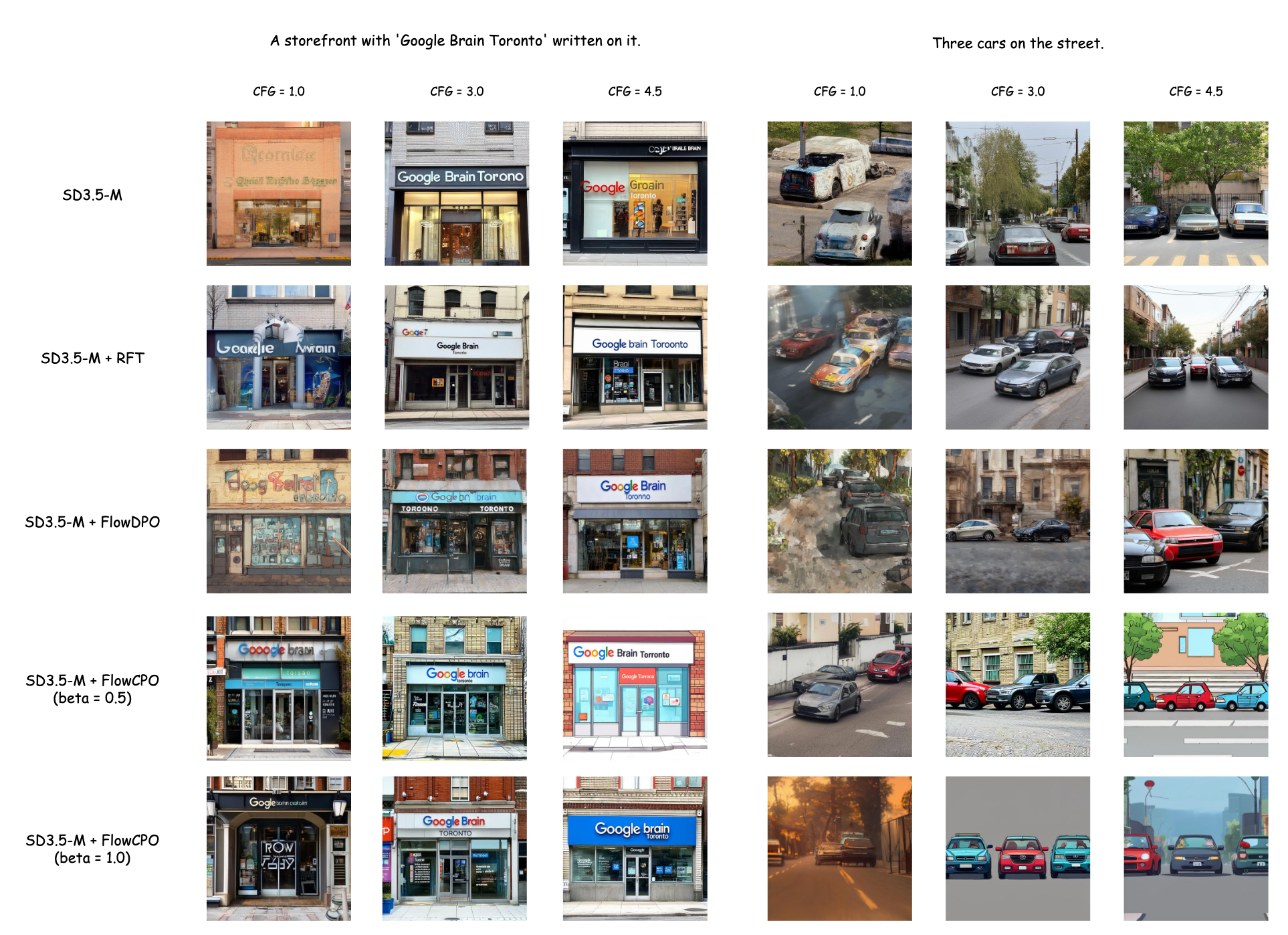}
    \caption{Qualitative comparison on the \texttt{DrawBench} benchmark. All samples are generated by models trained with \emph{out-of-domain} preference data.}
    \label{fig:flowcpo_image_quality_ood}
\end{figure}

\subsection{Additional Analyses Beyond Main Results}
\label{app:additional_analyses}

\subsubsection{Understanding the Role of Negative Regularization}
\label{app:negative_regularization}

\tabref{tab:regularization_mechanisms} shows a clear pattern within the \emph{GenEval-only preference training} setup: how negative samples enter the objective matters at least as much as whether they enter at all. Positive-only regularization provides a stable baseline and remains competitive across all three metric groups, especially under stronger CFG. Repulsive negative regularization is less reliable. Mild repulsion ($\lambda=-0.1$) stays close to positive-only training, but stronger repulsion ($\lambda=-1$) hurts most metrics, and overly aggressive repulsion ($\lambda=-10$) leads to numerical instability and collapsed training. This is consistent with the view that explicitly pushing the model away from dispreferred samples can over-amplify contrastive signals and destabilize optimization.

At moderate weights, \emph{attractive} negative regularization gives the best trade-off in this ablation between alignment strength and training stability. The configuration with $\beta=0.5$ and $\lambda=1$ achieves the highest GenEval score (0.84), whereas repulsion with $\lambda=-1$ degrades all reported metrics. The configuration with $\beta=1$ and $\lambda=0.1$ also attains the best OCR (0.59), ImgReward (1.14), and UniReward (3.23). Because all models in this table are trained with GenEval-only supervision, these gains on typography and general-preference metrics should be interpreted as cross-metric transfer rather than direct optimization. Both signs diverge at magnitude 10, so the evidence supports moderate attractive matching rather than a blanket stability claim.

\subsubsection{Hyperparameter Sensitivity and Ablation}
\label{app:ablation}

To study the hyperparameter sensitivity of the proposed objective function (\eqref{eq:final_loss}), we conduct ablations on the GenEval benchmark and track performance across training steps.

\begin{figure}[htbp]
    \centering
    \begin{subfigure}[b]{0.32\textwidth}
        \centering
        \includegraphics[width=\textwidth]{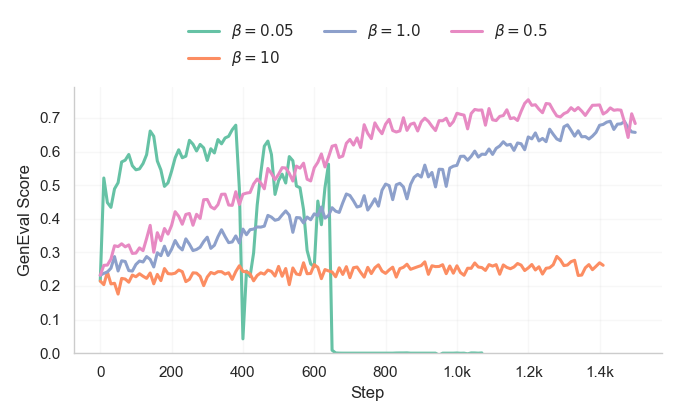}
        \caption{Effect of $\beta$}
        \label{fig:ablation_beta}
    \end{subfigure}
    \hfill
    \begin{subfigure}[b]{0.32\textwidth}
        \centering
        \includegraphics[width=\textwidth]{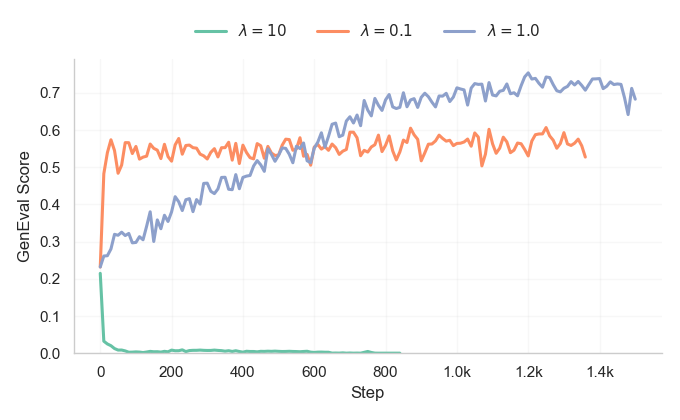}
        \caption{Effect of $\lambda$}
        \label{fig:ablation_lambda}
    \end{subfigure}
    \hfill
    \begin{subfigure}[b]{0.32\textwidth}
        \centering
        \includegraphics[width=\textwidth]{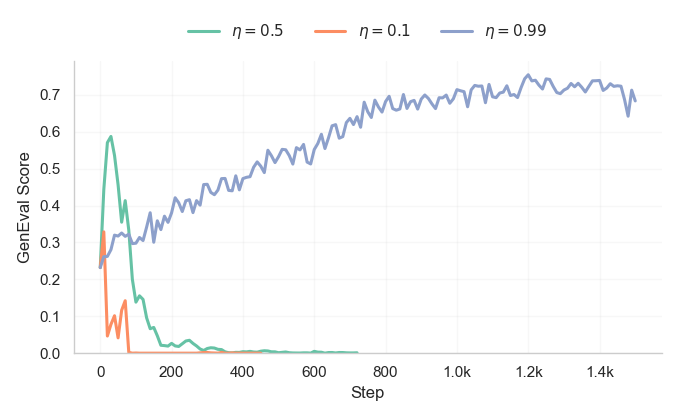}
        \caption{Effect of $\eta$}
        \label{fig:ablation_eta}
    \end{subfigure}
    \caption{\small
        \textbf{Hyperparameter sensitivity within the GenEval-only setup.} We monitor the GenEval score during training to analyze the impact of (a) the flow interpolation coefficient $\beta$, (b) the negative loss weight $\lambda$, and (c) the EMA parameter $\eta$. The figure is intended to show relative stability trends for the tested values.
    }
    \label{fig:ablation_studies}
\end{figure}

\paragraph{Flow Interpolation Coefficient ($\beta$).} 
$\beta$ controls how far the target flows move away from the reference prior. As shown in \figref{fig:ablation_beta}, moderate values ($\beta \in [0.5, 1.0]$) perform best in this sweep on GenEval. In particular, $\beta=0.5$ achieves the highest peak score.

\paragraph{Negative Regularization Weight ($\lambda$).} 
Balancing positive alignment and negative regularization is important. \figref{fig:ablation_lambda} shows that a balanced choice ($\lambda=1.0$) yields the most robust long-term performance among the tested values on GenEval.

\paragraph{Reference Prior Stability ($\eta$).} 
FlowCPO relies on a stable reference prior inside the flow-matching objective. \figref{fig:ablation_eta} shows that updating the reference model $v_{old}$ with a high-EMA decay rate ($\eta=0.99$) yields the most stable behavior in this sweep.

\subsection{Per-Prompt Diversity Evaluation}
\label{app:vendi_diversity}

\textbf{The primary empirical claim in RQ2 is that \method provides stable offline preference optimization with a forward-KL objective in the evaluated regime, not that it universally produces more diverse samples than FlowDPO.}
To evaluate sample diversity directly, we compute the Vendi Score~\citep{friedman2023vendi} separately over eight samples generated for each prompt and report the mean per-prompt score. We evaluate two classifier-free guidance (CFG) scales, 1.0 and 4.5. For each target column in \tabref{tab:vendi_diversity}, the quality score and Vendi Score are obtained from the corresponding specialist: GenEval-only fine-tuning for \texttt{GenEval}, OCR-only fine-tuning for \texttt{OCR}, and multi-reward fine-tuning for \texttt{PickScore, CLIPScore, HPSv2.1}. The main entry in each cell is the quality score, while the parenthesized entry is the Vendi Score. Higher is better for both.

\begin{table}[htbp]
    \centering
    \caption{Target quality and per-prompt diversity on \texttt{SD3.5-M}. Vendi Score is computed from eight samples per prompt and reported in parentheses. Each metric column evaluates the specialist fine-tuned for that target. Bold independently marks the highest quality score and the highest Vendi Score in each column.}
    \label{tab:vendi_diversity}
    \resizebox{\textwidth}{!}{
    \begin{tabular}{lcccc}
        \toprule
        \textbf{Method} & \textbf{CFG} & \textbf{GenEval $\uparrow$ (Vendi $\uparrow$)} & \textbf{OCR $\uparrow$ (Vendi $\uparrow$)} & \textbf{PickScore $\uparrow$ (Vendi $\uparrow$)} \\
        \midrule
        \multirow{2}{*}{SD3.5-M}
        & 1.0 & 0.24 (\textbf{2.1573}) & 0.12 (2.3073) & 20.51 (\textbf{2.2824}) \\
        & 4.5 & 0.63 (1.6540) & 0.59 (1.7778) & 22.34 (1.7742) \\
        \midrule
        \multirow{2}{*}{RFT}
        & 1.0 & 0.59 (1.8210) & 0.35 (2.2417) & 21.91 (1.9740) \\
        & 4.5 & 0.75 (1.4547) & 0.72 (1.4984) & 22.57 (1.5544) \\
        \midrule
        \multirow{2}{*}{FlowDPO}
        & 1.0 & 0.59 (1.8926) & 0.51 (\textbf{2.5117}) & 20.72 (1.8671) \\
        & 4.5 & 0.81 (1.5594) & 0.75 (1.6310) & \textbf{22.89} (1.5966) \\
        \midrule
        \multirow{2}{*}{\method ($\beta=0.5$)}
        & 1.0 & 0.76 (1.7385) & 0.83 (2.0578) & 22.48 (1.8936) \\
        & 4.5 & \textbf{0.82} (1.4773) & \textbf{0.86} (1.4955) & 22.69 (1.5251) \\
        \bottomrule
    \end{tabular}
    }
\end{table}

\paragraph{Results.}
\textbf{The results show a consistent quality--diversity trade-off: stronger CFG improves target quality while reducing per-prompt diversity.}
Across all comparisons in \tabref{tab:vendi_diversity}, increasing CFG from 1.0 to 4.5 raises the target quality score and lowers the corresponding Vendi Score. At CFG 1.0, \method achieves the highest quality among the compared methods for all three targets (0.76 GenEval, 0.83 OCR, and 22.48 PickScore), while retaining higher diversity than its own CFG-4.5 outputs. In particular, lowering CFG from 4.5 to 1.0 increases the Vendi Score of \method by 0.2612, 0.5623, and 0.3685 on the three targets, respectively, with corresponding quality decreases of 0.06, 0.03, and 0.21. However, \method does not maximize raw Vendi Score: the pretrained model is most diverse in the GenEval and PickScore columns.

\section{Broader Impacts and Release Considerations}
\label{app:broader_release}

\paragraph{Potential positive impacts.}
By turning preference alignment into a strictly offline optimization problem, our method can reduce the need for repeated online rollouts during alignment and make experimentation on continuous generative models more compute-efficient. Better alignment on structure-sensitive tasks such as compositional generation and text rendering can also improve the controllability and practical usefulness of text-to-image systems in benign applications such as design ideation, educational content creation, and accessibility-related graphics.

\paragraph{Potential negative impacts.}
The same improvement in preference alignment can increase the capability of image generators to produce more convincing synthetic content, including misleading text-in-image content or other deceptive media. In addition, the method can inherit biases from offline preference data and from the automatic reward signals used to construct winner/loser pairs, and the out-of-domain results in \secref{sec:experiments} together with the limitations discussed in \secref{sec:conclusion} indicate that behavior may degrade when the offline data distribution differs from the reference model.

\paragraph{Release considerations.}
We do not release new model checkpoints, a new dataset, or a code package with this preprint; the paper describes the training objective and experimental protocol only. The experiments rely on existing pretrained models, public benchmarks, and a public OOD preference dataset under their original access conditions, and any future release of code or checkpoints should preserve the upstream licenses, terms of use, and usage restrictions of those assets.

\end{document}